%% file: arxiv.tex
\documentclass[11pt]{article}
\usepackage{bm,amsmath,amsthm,amssymb,multicol,algorithmic,algorithm}

\usepackage{subcaption}
\usepackage{wrapfig}

\usepackage[round]{natbib}
\usepackage[colorlinks=true,citecolor=blue]{hyperref}

\usepackage{url}
\usepackage{multicol}
\usepackage{graphicx}
\usepackage[top=1in, left=1in, right=1in, bottom=1in]{geometry}

\usepackage[T1]{fontenc}

\usepackage[utf8]{inputenc} % allow utf-8 input
\usepackage[T1]{fontenc}    % use 8-bit T1 fonts
\usepackage{hyperref}       % hyperlinks
\usepackage{url}            % simple URL typesetting
\usepackage{booktabs}       % professional-quality tables
\usepackage{amsfonts}       % blackboard math symbols
\usepackage{nicefrac}       % compact symbols for 1/2, etc.
\usepackage{microtype}      % microtypography
\usepackage{graphicx}
\usepackage{multirow}
\usepackage{booktabs}
\usepackage{color}
\usepackage{bm}
\usepackage{bbm}
\usepackage{amsmath}
\usepackage{amsthm}
\usepackage{amssymb}
\usepackage{xcolor} 
\usepackage{pdfpages}
\makeatletter
\newtheorem*{rep@theorem}{\rep@title}
\newcommand{\newreptheorem}[2]{%
	\newenvironment{rep#1}[1]{%
		\def\rep@title{#2~\ref{##1}}%
		\begin{rep@theorem}}%
		{\end{rep@theorem}}}
\makeatother

\newtheorem{assumptionA}{A\!\!}

\newtheorem{theorem}{Theorem}
\newreptheorem{theorem}{Theorem}
\newtheorem{lemma}{Lemma}
\newreptheorem{lemma}{Lemma}

\input{preamble/packages.tex}


\newtheorem*{Lemma*}{Lemma}
\newtheorem*{Theorem*}{Theorem}

\allowdisplaybreaks

\begin{document}

\title{\bf On the Convergence of Decentralized Adaptive Gradient Methods}

\author{\textbf{Xiangyi Chen, Belhal Karimi, Weijie Zhao, Ping Li}\\\\
Cognitive Computing Lab\\
Baidu Research\\
10900 NE 8th St. Bellevue, WA 98004, USA\\
\texttt{\{xiangyichen1900, belhal.karimi, zhaoweijie12, pingli98\}@gmail.com}
}

\date{\vspace{0.4in}}

\maketitle

\begin{abstract}\vspace{0.1in}
\noindent\footnote{The work of Xiangyi Chen was conducted while he was an intern at Baidu Research -- Bellevue  in Summer 2019.}Adaptive gradient methods including Adam, AdaGrad, and their variants have been very successful for training deep learning models, such as neural networks. 
Meanwhile, given the need for distributed computing, distributed optimization algorithms are rapidly becoming a focal point.
With the growth of computing power and the need for using machine learning models on mobile devices, the communication cost of distributed training algorithms needs careful consideration.  In this paper, we introduce novel convergent decentralized adaptive gradient methods and rigorously incorporate adaptive gradient methods into decentralized training procedures. 
Specifically, we propose a general algorithmic framework that can convert existing adaptive gradient methods to their decentralized counterparts. 
In addition, we thoroughly analyze the convergence behavior of the proposed algorithmic framework and show that if a given adaptive gradient method converges, under some specific conditions, then its decentralized counterpart is also convergent. 
We illustrate the benefit of our generic decentralized framework on a prototype method, i.e., AMSGrad, both theoretically and numerically.
\end{abstract}

\newpage

\section{Introduction}

Distributed training of machine learning models has been drawing growing attention in the past few years due to its practical benefits and necessities. 
Given the evolution of computing capabilities of CPUs and GPUs, computation time in distributed settings is gradually dominated by the communication time in many circumstances~\citep{chilimbi2014project, mcmahan2017communication}. 
As a result, a large number of recent works have been focusing on reducing communication cost for distributed learning~\citep{alistarh2017qsgd,lin2017deep,wangni2018gradient,stich2018sparsified,wang2018atomo,tang2019doublesqueeze}. 
In the traditional parameter (central) server setting, where a parameter server is employed to manage communication in the whole network~\citep{Proc:Zhao_MLSys20}, many effective communication reductions have been proposed based on gradient compression~\citep{aji2017sparse} and quantization~\citep{chen2010approximate,jegou2010product,ge2013optimized,Proc:Xu_SIGMOD21} techniques. 
Despite these communication reduction techniques, its cost still, usually, scales linearly with the number of workers. 
Due to this limitation and with the sheer size of decentralized devices, the \emph{decentralized training paradigm}~\citep{duchi2011dual}, where the parameter server is removed and each node only communicates with its neighbors, is drawing attention. 
It has been shown in~\cite{lian2017can} that decentralized training algorithms can outperform parameter server-based algorithms when the training bottleneck is the communication cost. 
The decentralized paradigm is also preferred when a central parameter server is not available. 

In light of recent advances in nonconvex optimization, an effective way to accelerate training is by using adaptive gradient methods like AdaGrad~\citep{duchi2011adaptive}, Adam~\citep{kingma2014adam} or AMSGrad~\citep{reddi2019convergence}. 
Their popularity are due to their practical benefits in training neural networks, featured by faster convergence and ease of parameter tuning compared with Stochastic Gradient Descent (SGD)~\citep{robbins1951stochastic}.
Despite a large number of studies within the distributed optimization literature, few works have considered bringing adaptive gradient methods into distributed training, largely due to the lack of understanding of their convergence behaviors. 
Notably,~\citet{reddi2020adaptive} develop a decentralized ADAM method for distributed optimization problems with a direct application to federated learning.
An inner loop is employed to compute mini-batch gradients on each node and a global adaptive step is applied to update the global parameter at each outer iteration.
Yet, in the settings of our paper, nodes can only communicate \emph{to their neighbors} on a fixed communication graph while a server/worker communication is required in~\cite{reddi2020adaptive}.
Designing adaptive methods in such settings is highly non-trivial due to the already complex update rules and to the interaction between the effect of using adaptive learning rates and the decentralized communication protocols.
This paper is an attempt at bridging the gap between both realms in nonconvex optimization. 
Our \textbf{contributions} are summarized as follows:
\begin{itemize}
\item We investigate the use of  adaptive gradient methods in the decentralized training paradigm, where nodes have only a local view of the whole communication graph. We develop a general technique that converts an adaptive gradient method from a centralized method to its decentralized variant and highlight the importance of adaptive learning rate consensus. 
\item By using our proposed technique, we present a new decentralized optimization algorithm, called decentralized AMSGrad, as the decentralized counterpart of AMSGrad.
\item We provide a theoretical verification interface, in Theroem~\ref{thm: dagm_converge}, for analyzing the behavior of decentralized adaptive gradient methods obtained as a result of our technique.
Thus, we characterize the convergence rate of decentralized AMSGrad, which is the first convergent decentralized adaptive gradient method, to the best of our knowledge.
\end{itemize}
%A \emph{novel technique} in our framework is a mechanism to enforce a \emph{consensus on adaptive learning rates} at different nodes. We show the importance of consensus on adaptive learning rates by proving a divergent problem instance for a recently proposed decentralized adaptive gradient method, namely DADAM~\citep{nazari2019dadam}, a decentralized version of AMSGrad. 
%Though consensus is performed on the model parameter, DADAM lacks consensus principles on adaptive learning rates. After having presented existing related work and important concepts of decentralized adaptive methods I
The paper is organized as follows. In Section~\ref{sec:prelim}, we show the importance of adaptive learning rate consensus by proving a divergent example for a recently proposed decentralized adaptive gradient method, DADAM~\citep{nazari2019dadam}. In Section~\ref{sec:main}, we develop our general framework for converting adaptive gradient methods into their decentralized counterparts along with convergence analysis and converted algorithms. Illustrative experiments are presented in Section~\ref{sec:numerical}. 
Section~\ref{sec:conclusion} concludes~our~work.

\vspace{0.1in}

 \noindent\textbf{Notations}: $x_{t,i}$ denotes variable $x$ at node $i$ and iteration $t$. $\|\cdot \|_{abs}$ denotes the entry-wise $L_1$ norm of a matrix, i.e., $\|A\|_{abs}= \sum_{i,j} |A_{i,j}|$. 
We introduce important notations used throughout the paper: for any $t>0$, $G_t := [g_{t,N}]$ where $[g_{t,N}]$ denotes the matrix $[g_{t,1}, g_{t,2}, \cdots, g_{t,N}]$ (where $g_{t,i}$ is a column vector), $M_t := [m_{t,N}]$, $X_t := [x_{t,N}]$, $\overline {\nabla f}(X_t) := \frac{1}{N}\sum_{i=1}^N \nabla f_i(x_{t,i})$, $U_t := [u_{t,N}]$, $ \tilde U_t := [\tilde u_{t,N}]$, $ V_t := [ v_{t,N}]$, $\hat V_t := [\hat v_{t,N}]$, $\overline X_t := \frac{1}{N}\sum_{i=1}^N x_{t,i} $, $\overline U_t := \frac{1}{N}\sum_{i=1}^N u_{t,i} $ and $\overline {\tilde U_t} := \frac{1}{N}\sum_{i=1}^N  \tilde u_{t,i} $.

\section{Decentralized Adaptive Training and Divergence of DADAM}\label{sec:prelim}

\subsection{Related Work}

\textbf{Decentralized optimization.}\  
Traditional decentralized optimization methods include well-know algorithms such as ADMM~\citep{boyd2011distributed}, Dual Averaging~\citep{duchi2011dual}, Distributed Subgradient Descent~\citep{nedic2009distributed}. 
More recent algorithms include Extra~\citep{shi2015extra}, Next~\citep{di2016next}, Prox-PDA~\citep{hong2017prox}, GNSD~\citep{lu2019gnsd}, and Choco-SGD~\citep{koloskova2019decentralized}.  
While these algorithms are commonly used in applications other than deep learning, recent algorithmic advances in the machine learning community have shown that decentralized optimization can also be useful for training deep models such as neural networks. 
\citet{lian2017can} demonstrate that a stochastic version of Decentralized Subgradient Descent can outperform parameter server-based algorithms when the communication cost is high. 
\citet{tang2018d} propose the D$^2$ algorithm improving the convergence rate over Stochastic Subgradient Descent.
\citet{assran2019stochastic} propose the Stochastic Gradient Push that is more robust to network failures for training neural networks. 
The study of decentralized training algorithms in the machine learning community is only at its initial stage. 
No existing work, to our knowledge, has seriously considered integrating \emph{adaptive gradient methods} in the setting of decentralized learning.
One noteworthy work~\citep{nazari2019dadam} proposes a decentralized version of AMSGrad~\citep{reddi2019convergence} and it is proven to satisfy some non-standard regret.

\vspace{0.1in}
\noindent\textbf{Adaptive gradient methods.}\   
Adaptive gradient methods have been popular in recent years due to their superior performance in training neural networks. 
Most commonly used adaptive methods include AdaGrad~\citep{duchi2011adaptive} or Adam~\citep{kingma2014adam} and their variants.  
Key features of such methods lie in the use of momentum and adaptive learning rates (which means that the learning rate is changing during the optimization and is anisotropic, i.e., depends on the dimension).
The method of reference, called Adam, has been analyzed in~\cite{reddi2019convergence} where the authors point out an error in previous convergence analyses. 
Since then, a variety of papers have been focusing on analyzing the convergence behavior of the numerous existing adaptive gradient methods. 
\citet{ward2019adagrad},~\citet{li2019convergence} derive convergence guarantees for a variant of AdaGrad without coordinate-wise learning rates. 
\citet{chen2018convergence} analyze the convergence behavior of a broad class of algorithms including AMSGrad and AdaGrad. \citet{zhou2018convergence} give a more refined analysis of AMSGrad with better convergence rate.
\citet{zou2018convergence} provide a unified convergence analysis for AdaGrad with momentum.
Noticeable recent works on adaptive gradient methods can be found in~\cite{agarwal2019efficient,luo2019adaptive,zaheer2018adaptive}.

\subsection{Decentralized Optimization }

In distributed optimization (with $N$ nodes), we aim at solving the following problem
\begin{align}\label{eq:minproblem}
\min_{x \in \mathbb{R}^d} \frac{1}{N}\sum_{i=1}^N f_i(x) \, ,
\end{align}
where $x$ is the vector of parameters and $f_i$ is only accessible by the $i$-th node. 
Through the prism of empirical risk minimization procedures, $f_i$ can be viewed as the average loss of the data samples located at node $i$, for $i \in [N]$. 
Throughout the paper, we make the following mild assumptions required for analyzing the convergence behavior of the different decentralized optimization algorithms introduced above:
\begin{assumptionA}\label{a:diff}
For all $i \in [N]$, $f_i$ is differentiable and the gradients are $L$-Lipschitz, i.e., for all $(x, y) \in \mathbb{R}^d$, $\|\nabla f_i(x) - \nabla f_i(y) \| \leq L\|x-y\|$.
\end{assumptionA}
\begin{assumptionA}\label{a:boundsto}
We assume that, at iteration $t$, node $i$ accesses a stochastic gradient $g_{t,i}$. The stochastic gradients and the gradients of $f_i$ have bounded $L_{\infty}$ norms, i.e., $\|g_{t,i}\| \leq G_{\infty}$, $\|\nabla f_i(x)\|_{\infty} \leq G_{\infty}$. 
\end{assumptionA}
\begin{assumptionA}\label{a:boundedvar}
The gradient estimators are unbiased and each coordinate has bounded variance, i.e., $\mathbb E [g_{t,i}] = \nabla f_i(x_{t,i}) $ and $\mathbb E [([g_{t,i} - f_i(x_{t,i})]_j)^2] \leq  \sigma^2, \forall t,i,j$ . 
\end{assumptionA}
Assumptions A\ref{a:diff} and A\ref{a:boundedvar} are standard in the distributed optimization literature. 
A\ref{a:boundsto} is slightly stronger than the traditional assumption stating that the estimator has bounded variance, yet, it is commonly used for the analysis of adaptive gradient methods~\citep{chen2018convergence,ward2019adagrad}. 
Note that the bounded gradient estimator assumption A\ref{a:boundsto} implies the bounded variance assumption A\ref{a:boundedvar}.
% We willingly denote the variance bound and the estimator bound differently to avoid confusion when we use them for different purposes. 
In decentralized optimization, the nodes are connected as a graph and each node only communicates to its neighbors. 
Hence, one usually constructs an $N \times N$ matrix $W$ for information sharing when designing new training algorithms. 
We denote by $\lambda_i$ its $i$-th largest eigenvalue and define $\lambda \triangleq \max (|\lambda_2|,|\lambda_N|)$.
The matrix $W$ cannot be arbitrary, its required key properties are listed in the following assumption:
\begin{assumptionA}\label{a:matrixW}
The matrix $W$ satisfies: \textsc{(i)} $\sum_{j=1}^N W_{i,j} = 1$,  $\sum_{i=1}^N W_{i,j} = 1$, $W_{i,j} \geq 0$, \textsc{(ii)} $\lambda_1 = 1$, $|\lambda_2| < 1$, $|\lambda_N| < 1 $ and \textsc{(iii)} $W_{i,j} = 0 $ if node $i$ and node $j$ are not neighbors.
\end{assumptionA}
%  Throughout this paper, we will assume A\ref{a:diff}-A\ref{a:matrixW} hold.
 We now present the convergence failure of current decentralized adaptive method before introducing our general framework for decentralized adaptive gradient methods.

\subsection{Divergence of DADAM}\label{sec:divergence}

Recently,~\citet{nazari2019dadam} initiated an attempt to bring adaptive gradient methods into decentralized optimization with Decentralized ADAM (DADAM), shown in Algorithm~\ref{alg: dadam}.
DADAM is essentially a decentralized version of ADAM and the key modification is the use of a consensus step on the optimization variable $x$ to transmit information across the network, encouraging~its~convergence.

\begin{algorithm}[H]
	\caption{DADAM (with N nodes)}
	\label{alg: dadam}
	\begin{algorithmic}[1]
		\STATE {\bfseries Input:} $\alpha$, current point $X_t$, $u_{\frac{1}{2},i} = \hat v_{0,i} = \epsilon \mathbf{1}$, $m_0=0$ and mixing matrix $W$
		\FOR {$t = 1,2,\cdots, T$}
		\STATE \textbf{for all }$i \in [N]$ \textbf{do in parallel}
% 		\STATE In parallel for each worker $i \in [N]$
		\STATE \quad $g_{t,i}  \leftarrow \nabla f_i(x_{t,i}) + \xi_{t,i}$
		\STATE \quad $m_{t,i} = \beta_1 m_{t-1,i} + (1-\beta_1) g_{t,i}$ 
		\STATE \quad $v_{t,i} = \beta_2 v_{t-1,i}+(1-\beta_2)g_{t,i}^2$
		\STATE \quad $\hat v_{t,i} = \beta_3 \hat v_{t,i} + (1-\beta_3) \max(\hat v_{t-1,i},v_{t,i})$
		\STATE \quad $x_{t+\frac{1}{2},i} = \sum_{j=1}^N W_{ij}x_{t,j}$
		\STATE \quad $x_{t+1,i} = x_{t+\frac{1}{2},i} - \alpha \frac{m_{t,i}}{\sqrt{\hat v_{t,i}}}$
		\ENDFOR
	\end{algorithmic}
\end{algorithm}

The matrix $W$ is a doubly stochastic matrix (which satisfies A\ref{a:matrixW}) employed for achieving average consensus of $x$. 
Introducing such mixing matrix is standard while designing the extension of an algorithm to its decentralized variant, such as distributed gradient descent~\citep{nedic2009distributed, yuan2016convergence}. 
It is proven in~\cite{nazari2019dadam} that DADAM admits a non-standard regret bound in the online setting. Nevertheless, whether the algorithm can converge to stationary points in standard offline settings such training neural networks is still unknown.
The next theorem shows that DADAM may fail to converge in the offline  settings.

\begin{theorem}\label{thm: dadam_diverge}
There exists a problem satisfying A\ref{a:diff}-A\ref{a:matrixW} where DADAM fails to converge to a stationary points with $\nabla f(\bar X_t) = 0$.   
\end{theorem}

\begin{proof}
% \textbf{Proof}: 
Consider a two-node setting with objective function $f(x) =1/2 \sum_{i=1}^2 f_i(x)$ and $f_1(x) =  \mathbbm 1[|x|\leq 1] 2x^2 +  \mathbbm 1[|x|> 1] (4|x|-2)$, $f_2(x) =  \mathbbm 1[|x-1|\leq 1](x-1)^2 + \mathbbm 1[|x-1| > 1] (2|x-1|-1)$. We set the mixing matrix  $W = [0.5,0.5;0.5,0.5]$. The optimal solution is $x^* = 1/3$.
Both $f_1$ and $f_2$ are smooth and convex with bounded gradient norm 4 and 2, respectively. 
We also have $L = 4$ (defined in A\ref{a:diff}). 
If we initialize with $x_{1,1} = x_{1,2} = -1$ and run DADAM with $\beta_1 = \beta_2 =\beta_3 = 0$ and $\epsilon \leq 1$, we will get $\hat v_{1,1} = 16$ and $\hat v_{1,2} = 4$. 
Since $|g_{t,1}| \leq 4, |g_{t,2}| \leq 2$ due to bounded gradient, and $(\hat v_{t,1},\hat v_{t,2})$ are non-decreasing, we have $\hat v_{t,1} = 16, \hat v_{t,2}=4, \forall t \geq 1$. 
Thus, after $t=1$, DADAM is equivalent to running decentralized gradient descent (D-PSGD)~\citep{yuan2016convergence} with a re-scaled $f_1$ and $f_2$, i.e.,  running D-PSGD on
$f'(x) = \sum_{i=1}^2 f_i'(x)$ with $f_1'(x) =  0.25 f_1(x)$ and $f_2'(x) = 0.5  f_2(x)$, which unique optimal $x'=0.5$. 
Define $\bar x_t = (x_{t,1}+x_{t,2})/2$, then by Theorem 2 in~\cite{yuan2016convergence}, we have when $\alpha < 1/4$, $f'(\bar x_t) - f(x') = O(1/(\alpha t))$. 
Since $f'$ has a unique optima $x'$, the above bound implies $\bar x_t$ is converging to $x'=0.5$ which has non-zero gradient on function $\nabla f(0.5) = 0.5$.
\end{proof}

Theorem~\ref{thm: dadam_diverge} shows that, even though DADAM is proven to satisfy some regret bounds~\citep{nazari2019dadam}, it can fail to converge to stationary points in the nonconvex offline setting (common for training neural networks). 
We conjecture that this inconsistency in the convergence behavior of DADAM is due to the definition of the regret in~\citet{nazari2019dadam}. We want to remark that this is not the first time adaptive gradient methods are found to be divergent. 
For example, \citet{reddi2019convergence} constructs examples showing that ADAM is divergent and \citet{chen2020toward} exhibits a naive application of adaptive gradient methods under the federated learning settings that can potentially fail to converge. 
All these examples contribute to our motivation to rigorously study the convergence of adaptive gradient methods in the decentralized setting. 
The next section presents decentralized adaptive gradient methods that are guaranteed to converge to stationary points under  assumptions and provide a characterization of that convergence in finite-time and independently of the initialization.

\section{On the Convergence of Decentralized Adaptive Gradient Methods}\label{sec:main}

In this section, we discuss the difficulties of designing adaptive gradient methods in decentralized optimization and introduce an algorithmic framework that can turn existing convergent adaptive gradient methods into their decentralized counterparts. 
We also develop the first convergent decentralized adaptive gradient method, converted from AMSGrad, \emph{as an instance of this  framework}.

\subsection{Importance and Difficulties of Consensus on Adaptive Learning Rates}

The divergent example provided in the previous section implies that one should synchronize the adaptive learning rates on different nodes. 
This can easily be achieved in the parameter server setting where all the nodes are sending their gradients to a central server at each iteration.
The parameter server can then exploit the received gradients to maintain a sequence of synchronized adaptive learning rates when updating the parameters, see~\cite{reddi2020adaptive} for further details.
However, in our decentralized setting, every node can only communicate with its neighbors and such central server does not exist.
Under that setting, the information for updating the adaptive learning rates can only be shared locally instead of broadcasted over the whole network.
This makes it impossible to obtain, in a single iteration, a synchronized adaptive learning rate update using all the information in the network. 

\textit{Systemic Approach:} 
On a systemic level, one way to alleviate this bottleneck is to design communication protocols in order to give each node access to the same aggregated gradients over the whole network, at least periodically if not at every iteration.
Therefore, the nodes can update their individual adaptive learning rates based on the same shared information. 
However, such solution may introduce an extra communication cost since it involves broadcasting the information over the whole network.

\textit{Algorithmic Approach:} 
Our contributions being on an algorithmic level, another way to solve the aforementioned problem is by letting the sequences of adaptive learning rates, present on different nodes, to gradually \emph{consent}, through the iterations. 
Intuitively, if the adaptive learning rates can consent fast enough, the difference among the adaptive learning rates on different nodes will not affect the convergence behavior of the algorithm.  
Consequently, no extra communication costs need to be introduced.
We now develop this exact idea within the existing adaptive methods stressing on the need for a relatively low-cost and easy-to-implement consensus of adaptive learning rates. 

Below is main archetype of the adaptive rates consensus mechanism within a decentralized framework that we propose in this paper.

\subsection{Unifying Decentralized Adaptive Gradient Framework}

While each node can have different $\hat v_{t,i}$ in DADAM (Algorithm~\ref{alg: dadam}), one can keep track of the min/max/average of these adaptive learning rates and use that latter quantity as the new adaptive learning rate. 
The upstream definition of some convergent lower and upper bounds may also lead to a gradual synchronization of the adaptive learning rates on different nodes as developed for AdaBound in~\cite{luo2019adaptive}.

\begin{algorithm}[H]
	\caption{Decentralized Adaptive Gradient Method (with N nodes)}
	\label{alg: dadaptive}
	\begin{algorithmic}[1]
		\STATE {\bfseries Input:}  $\alpha$, initial point $x_{1,i} = x_{init}, u_{\frac{1}{2},i} = \hat v_{0,i}, m_{0,i}=0$, mixing matrix $W$
		\FOR{$t=1,2,\cdots,T$}
		\STATE \textbf{for all }$i \in [N]$ \textbf{do in parallel}
		\STATE \quad $g_{t,i}  \leftarrow \nabla f_i(x_{t,i}) + \xi_{t,i}$
		\STATE \quad $m_{t,i} = \beta_1 m_{t-1,i} + (1-\beta_1) g_{t,i}$ 
		\STATE  \quad $\hat v_{t,i} = r_t(g_{1,i},\cdots,g_{t,i})$
		\STATE \quad $x_{t+\frac{1}{2},i} = \sum_{j=1}^N W_{ij}x_{t,j}$
	    \STATE \quad $\tilde u_{t,i} = \sum_{j=1}^N W_{ij} \tilde u_{t-\frac{1}{2},j}$
	    \STATE  \quad $u_{t,i} = \max(\tilde u_{t,i}, \epsilon)$
		\STATE \quad $x_{t+1,i} = x_{t+\frac{1}{2},i} - \alpha \frac{m_{t,i}}{\sqrt{u_{t,i}}}$
		\STATE \quad $\tilde u_{t+\frac{1}{2},i} = \tilde u_{t,i} - \hat v_{t-1,i} + \hat v_{t,i}$
		\ENDFOR
	\end{algorithmic}
\end{algorithm}

In this paper, we present an algorithm framework for decentralized adaptive gradient methods as Algorithm~\ref{alg: dadaptive}, which uses average consensus of $\hat v_{t,i}$ (see consensus update in line 8 and 11) to help convergence.  
Algorithm~\ref{alg: dadaptive} can become different adaptive gradient methods by specifying $r_t$ as different functions. E.g., when we choose $\hat v_{t,i} = {\frac{1}{t}\sum_{k=1}^t g_{k,i}^2}$ , Algorithm~\ref{alg: dadaptive} becomes a decentralized version of AdaGrad. When one chooses $\hat v_{t,i}$ to be the adaptive learning rate for AMSGrad, we get decentralized AMSGrad (Algorithm~\ref{alg: damsgrad}).
The intuition of using average consensus is that for adaptive gradient methods such as AdaGrad or Adam, $\hat v_{t,i}$ approximates the second moment of the gradient estimator, the average of the estimations of those second moments from different nodes is an estimation of second moment on the whole network.  
Also, this design will not introduce any extra hyperparameters that can potentially complicate the tuning process ($\epsilon$ in line 9 is important for numerical stability as in vanilla Adam).  
The following result gives a finite-time convergence rate for our framework described in Algorithm~\ref{alg: dadaptive}.

\begin{theorem}\label{thm: dagm_converge}
Assume A\ref{a:diff}-A\ref{a:matrixW}. %Set $\alpha = 1/\sqrt{Td}$.
	When $\alpha  \leq \frac{\epsilon^{0.5}}{16L} $, 
	  Algorithm~\ref{alg: dadaptive} yields the following regret bound
	  {\small
	\begin{align}\label{eq: thm11}
\frac{1}{T}\sum_{t=1}^T  \mathbb E \left [\left\|\frac{\nabla f( \overline X_{t})}{\overline U_{t}^{1/4}}\right\|^2  \right]
&	\leq   C_1\left(\frac{1}{T\alpha} ( \mathbb E  [f( Z_{1})]  -  \min_x  f(x)) +  \alpha  \frac{d\sigma^2}{N}\right) +  C_2 \alpha^2 d 
	\nonumber \\
    &+ C_3 \alpha^3d  + \frac{1}{T\sqrt{N}} (C_4 +  C_5 \alpha)  \mathbb E \left[ \sum_{t=1}^{T}   \|    (- \hat V_{t-2} + \hat V_{t-1} ) \|_{abs} \right] 
	\end{align}
	}%
where $\| \cdot\|_{abs}$  denotes the entry-wise $L_1$ norm of a matrix (i.e $\| A\|_{abs} = \sum_{i,j}{|A_{ij}|}$). The constants $C_1 =  \max (4, 4{L/\epsilon})$,
	$C_2 =  6 (( \beta_1/(1-\beta_1))^2 + 1/(1-\lambda)^2 )L  G_{\infty}^2 /\epsilon^{1.5}$,
	$C_3 =  16L^2  (1-\lambda) G_{\infty}^2/\epsilon^2$,
	$C_4 =   2/ (\epsilon^{1.5}(1-\lambda)) (     \lambda + \beta_1/(1-\beta_1)) G_{\infty}^2$,
	$C_5 =   2/ (\epsilon^{2}(1-\lambda))   L  (\lambda + \beta_1/(1-\beta_1)) G_{\infty}^2  + 4/ (\epsilon^{2}(1-\lambda))   L    G_{\infty}^2$ are independent of $d$, $T$ and $N$. In addition, $\frac{1}{N}\sum_{i=1}^N\left\| {  x_{t,i} -   \overline X_{t}}  \right\|^2   \leq \alpha^2 \left (\frac{1}{1-\lambda} \right)^2 d G_{\infty}^2 \frac{1}{\epsilon}$ which quantifies the consensus error.
\end{theorem}

\newpage

In addition, one can specify $\alpha$ to show convergence in terms of $T$, $d$, and $N$. 
An immediate result, shown in Corollary~\ref{corl: adm_convergence}, is by setting $\alpha = \sqrt{N}/\sqrt{Td}$:
\begin{corollary}\label{corl: adm_convergence}
Assume A\ref{a:diff}-A\ref{a:matrixW}. Set $\alpha = \sqrt{N}/\sqrt{Td}$.
When $\alpha  \leq \frac{\epsilon^{0.5}}{16L} $, Algorithm~\ref{alg: dadaptive} yields:
 {\small
	\begin{align}\label{eq: thm1}
	\frac{1}{T}\sum_{t=1}^T  \mathbb E \left [\left\|\frac{\nabla f( \overline X_{t})}{\overline U_{t}^{1/4}}\right\|^2  \right]
&	\leq   C_1 \frac{\sqrt{d}}{\sqrt{TN}} \left(( \mathbb E  [f( Z_{1})]  -  \min_x  f(x)) +    \sigma^2 \right)  +  C_2 \frac{N}{T} \nonumber \\
	&+  C_3 \frac{N^{1.5}}{T^{1.5}d^{0.5}} 
    +  \left(C_4 \frac{1}{T\sqrt{N}} +  C_5   \frac{1}{T^{1.5}d^{0.5}}\right) \mathbb E \left[\mathcal V_T \right] 
	\end{align}
	}
% {\small
%	\begin{align}\label{eq: thm1}
%	\frac{1}{T}\sum_{t=1}^T  \mathbb E \left [\left\|\frac{\nabla f( \overline X_{t})}{\overline U_{t}^{1/4}}\right\|^2  \right]
%	\leq    \frac{C_1\sqrt{d} (( \mathbb E  [f( Z_{1})]  -  \min_x  f(x)) +    \sigma^2 )}{\sqrt{TN}}  +   \frac{C_2 N}{T} +   \frac{C_3 N^{1.5}}{T^{1.5}d^{0.5}} 
%    +  \left( \frac{C_4}{T\sqrt{N}} +     \frac{C_5}{T^{1.5}d^{0.5}}\right) \mathbb E \left[\mathcal V_T \right] 
%	\end{align}
%	}
	where $ \mathcal{V}_T : = \sum_{t=1}^{T}   \|    (- \hat V_{t-2} + \hat V_{t-1} ) \|_{abs}$ and $C_1$, $C_2$, $C_3$, $C_4$,  $C_5$ are defined in Theorem~\ref{thm: dagm_converge}.
\end{corollary}
Corollary~\ref{corl: adm_convergence} indicates that if  $\mathbb E [\mathcal{V}_T  ]  = o(T)$ and $\bar U_t$ is bounded from above, then Algorithm~\ref{alg: dadaptive} is guaranteed to converge to stationary points of the loss function. 
Intuitively, this means that if the adaptive learning rates on different nodes do not change too fast, the algorithm can converge. In convergence analysis, the term $  \mathbb E [\mathcal{V}_T  ] $ upper bounds the total bias in update direction caused by the correlation between $m_{t,i}$ and $\hat v_{t,i}$.
It is shown in~\cite{chen2018convergence} that when $N=1$, $  \mathbb E [\mathcal{V}_T  ]  = \tilde O (d) $ for AdaGrad and AMSGrad. Besides, $\mathbb E [\mathcal{V}_T  ]  = \tilde O (Td) $ for Adam which do not converge.  {Later, we will show convergence of decentralized versions of AMSGrad and AdaGrad by bounding this term as $O(Nd)$ and $O(Nd \log (T))$, respectively.}
 %The intuition $  \mathbb E [\mathcal{V}_T  ]  = o(T)$ can guarantee divergence is that the correlation between $\hat v_{t,i}$ and $m_{t,i}$ (due to their shared dependency on historical gradients) can make update direction negatively correlated with true gradient in expectation, leading to a non-negligible bias in updates. However, the total bias across $T$ iterations introduced by such a correlation is bounded by the term $  \mathbb E [\mathcal{V}_T  ] $. Thus, if $  \mathbb E [\mathcal{V}_T  ] $ grows sublinearly with $T$, convergence can still be guaranteed. 
 Corollary~\ref{corl:  adm_convergence} also conveys the benefits of using more nodes in the graph employed. 
When $T$ is large enough such that the term $O(\sqrt{d}/\sqrt{TN})$ dominates the right hand side of \eqref{eq: thm1}, then linear speedup can be achieved by increasing the number of nodes $N$.   

Another point worth discussion is the choice of $W$ since the convergence rate depends on $\lambda$ which is depedent on $W$. A common way to set $W$ for undirected graph is the maximum-degree method (MDM) in \cite{boyd2004fastest}. Denote $d_i$ as degree of vertex $i$ and $d_{\max} = \max_i d_i$, MDM sets $W_{i,i} = 1-d_i/d_{\max}$, $W_{i,j} = 1/d_{\max}$ if $i\neq j$ and $(i,j)$ is an edge,  and $W_{i,j} = 0$ otherwise. This $W$ ensures Assumption A\ref{a:matrixW} for many common connected graph types, so does the variant $\gamma I + (1-\gamma) W$ for any $\gamma \in [0,1)$.  
A more refined choice of $W$ coupled with a comprehensive discussion on $\lambda$ in our Theorem~\ref{thm: dagm_converge} can be found in \cite{boyd2009fastest}, e.g., $1-\lambda =O(1/N^2)$ for cycle graphs, $1-\lambda =O(1/\log(N))$ for hypercube graphs, $\lambda = 0$ for fully connected graph. 
Intuitively, $\lambda$ can be close to 1 for sparse graphs and to 0 for dense graphs.
This is consistent \eqref{eq: thm11}, whose RHS is large for $\lambda$ close to 1 and small for $\lambda $ close to 0 since average consensus on sparser graphs is expected to take longer time.
%Indeed, as $N$ becomes larger, the term $\sigma^2/N$ will be small. 
%This is also strengthened by the fact that with a larger $N$, the training process tends to be more stable. 

\subsection{Application to AMSGrad algorithm}\label{sec:amsgrad}

We now present, in Algorithm~\ref{alg: damsgrad}, a notable special case of our algorithmic framework, namely Decentralized AMSGrad, which is a decentralized variant of AMSGrad.
Compared with DADAM, the above algorithm exhibits a dynamic average consensus mechanism to keep track of the average of $\{\hat v_{t,i}\}_{i=1}^N$, stored as $\tilde u_{t,i}$ on $i$-th node, and uses $u_{t,i} := \max(\tilde u_{t,i}, \epsilon)$ for updating the adaptive learning rate for $i$-th node. 
As the number of iteration grows, even though $\hat v_{t,i}$ on different nodes can converge to different constants, the $u_{t,i}$ will converge to the same number $ \lim \limits_{t \rightarrow \infty} \frac{1}{N} \sum_{i=1}^N\hat v_{t,i} $ if the limit exists. 

\begin{algorithm}[t]
	\caption{Decentralized AMSGrad (N nodes)}
	\label{alg: damsgrad}
	\begin{algorithmic}[1]
		\STATE {\bfseries Input:} learning rate $\alpha$, initial point $x_{1,i} = x_{init}, u_{\frac{1}{2},i} = \hat v_{0,i} = \epsilon \mathbf 1\ (\text{with } \epsilon \geq 0), m_{0,i}=0$, mixing matrix $W$ 
		\FOR{$t = 1,2,\cdots,T$}
		\STATE \textbf{for all }$i \in [N]$ \textbf{do in parallel}
		\STATE \quad $g_{t,i}  \leftarrow \nabla f_i(x_{t,i}) + \xi_{t,i}$
		\STATE \quad $m_{t,i} = \beta_1 m_{t-1,i} + (1-\beta_1) g_{t,i}$ 
		\STATE \quad  $ v_{t,i} = \beta_2 v_{t-1,i} + (1-\beta_2) g_{t,i}^2 $
		\STATE  \quad $\hat v_{t,i} = \max (\hat v_{t-1,i}, v_{t,i} )$
		\STATE \quad $x_{t+\frac{1}{2},i} = \sum_{j=1}^N W_{ij}x_{t,j}$
		\STATE \quad $\tilde u_{t,i} = \sum_{j=1}^N W_{ij}\tilde u_{t-\frac{1}{2},j}$
	    \STATE \quad $u_{t,i} = \max(\tilde u_{t,i}, \epsilon)$
		\STATE \quad $x_{t+1,i} = x_{t+\frac{1}{2},i} - \alpha \frac{m_{t,i}}{\sqrt{u_{t,i}}}$
		\STATE \quad $\tilde u_{t+\frac{1}{2},i} = \tilde u_{t,i} - \hat v_{t-1,i} + \hat v_{t,i}$
		\ENDFOR
	\end{algorithmic}
\end{algorithm}

This average consensus mechanism enables the consensus of adaptive learning rates on different nodes, which accordingly guarantees the convergence of the method to stationary points. 
The consensus of adaptive learning rates is the key difference between decentralized AMSGrad and DADAM and is the reason why decentralized AMSGrad is  convergent while DADAM is not.

%\begin{wrapfigure}[20]{r}{.5\linewidth}

\newpage

One may notice that decentralized AMSGrad does not reduce to AMSGrad for $N=1$ since the quantity $u_{t,i}$ in line 10 is calculated based on $v_{t-1,i}$ instead of $v_{t,i}$.
This design encourages the execution of gradient computation and communication in a parallel manner. 
Specifically, line 4-7 (line 4-6) in Algorithm~\ref{alg: damsgrad} (Algorithm~\ref{alg: dadaptive}) can be executed in parallel with line 8-9 (line 7-8) to overlap communication and computation time. 
If $u_{t,i}$ depends on $v_{t,i}$ which in turn depends on $g_{t,i}$, the gradient computation must finish before the consensus step of the adaptive learning rate in line 9. 
This can slow down the running time per-iteration of the algorithm. 
To avoid such delayed adaptive learning, adding $\tilde u_{t-\frac{1}{2},i} = \tilde u_{t,i} - \hat v_{t-1,i} + \hat v_{t,i}$ before line 9 and getting rid of line 12 in Algorithm~\ref{alg: dadaptive} is an option.
Similar convergence guarantees will hold since one can easily modify our proof of Theorem~\ref{thm: dagm_converge} for such update rule. 
As stated above, Algorithm~\ref{alg: damsgrad} converges, with the following rate:
\begin{theorem}\label{thm: dams_converge}
Assume A\ref{a:diff}-A\ref{a:matrixW}.
Set $\alpha = 1/\sqrt{Td}$. When $\alpha  \leq \frac{\epsilon^{0.5}}{16L} $, then Algorithm~\ref{alg: damsgrad} satisfies:
 
	  \begin{align}\notag
	  \frac{1}{T}\sum_{t=1}^T  \mathbb E \left [\left\|\frac{\nabla f( \overline X_{t})}{\overline U_{t}^{1/4}}\right\|^2  \right]
	  \leq  C_1' \frac{\sqrt{d}}{\sqrt{TN}} \left(D_f +    \sigma^2 \right) +  C_2' \frac{N}{T}  +  C_3' \frac{N^{1.5}}{T^{1.5}d^{0.5}} 
+  C_4' \frac{\sqrt{N}d}{T} +  C_5'  \frac{Nd^{0.5}}{T^{1.5}}  \ ,
	  \end{align}
	  
	where $D_f := \mathbb E  [f( Z_{1})]  -  \min_x  f(x)$, $C_1' = C_1$, $C_2' = C_2$, $C_3' = C_3$, $C_4' = C_4G_{\infty}^2$ and $C_5' = C_5 G_{\infty}^2 $. $C_1,C_2, C_3, C_4, C_5$ are independent of $d$, $T$ and $N$ defined in Theorem~\ref{thm: dagm_converge}. In addition, the consensus of variables at different nodes is given by $\frac{1}{N}\sum_{i=1}^N\left\| {  x_{t,i} -   \overline X_{t}}  \right\|^2   \leq \frac{N}{T} \left (\frac{1}{1-\lambda} \right)^2  G_{\infty}^2 \frac{1}{\epsilon}$. 
\end{theorem}

Theorem~\ref{thm: dams_converge} shows that Algorithm~\ref{alg: damsgrad} converges with a rate of  $\mathcal{O}(\sqrt{d}/\sqrt{T})$ when $T$ is large, which is the best known convergence rate under the given assumptions. 
Note that in some related works, SGD admits a convergence rate of $\mathcal{O}(1/\sqrt{T})$ without any dependence on the dimension of the problem.
Such improved convergence rate is derived under the assumption that the gradient estimator have a bounded $L_2$ norm, which can thus hide a dependency of $\sqrt{d}$ in the final convergence rate. {Another remark is the convergence measure can be converted to $\frac{1}{T}\sum_{t=1}^T  \mathbb E \left [\left\|{\nabla f( \overline X_{t})}\right\|^2  \right]$
using the fact that $\|\overline U_{t}\|_{\infty} \leq G_{\infty}^2$  (by update rule of Algorithm~\ref{alg: damsgrad}), for the ease of comparison with existing literature.}

\newpage

\noindent\textbf{Proof Sketch of Theorem~\ref{thm: dagm_converge}:} The detailed proofs are reported in the appendix of this paper.

\vspace{0.1in}

\textsl{Step 1: Reparameterization.} \hspace{0.01in} Similarly to~\cite{yan2018unified, chen2018convergence} with SGD (with momentum) and centralized adaptive gradient methods, define the following auxiliary sequence:
 $
 Z_{t} = \overline X_t + \frac{\beta_1}{1-\beta_1} (\overline X_t - \overline X_{t-1}) \, ,
 $
with $\overline X_{0} \triangleq \overline X_1$.
Such an auxiliary sequence can help us deal with the bias brought by the momentum and simplifies the convergence analysis. 

\vspace{0.1in}
 
 \textsl{Step 2: Bounding gradient.} \hspace{0.01in} With the help of $Z_t$, we can remove the complicated update dependence on $m_t$, and perform convergence analysis to bound gradient of $Z_t$. Then bound gradient of $\overline X_t$ by smoothness of gradient, which yields:
  
 \begin{align} \label{eq: exp_telescope_sketchmain}
 \frac{1}{T}\sum_{t=1}^T  \mathbb E \left [\left\|\frac{\nabla f( \overline X_{t})}{\overline U_{t}^{1/4}}\right\|^2  \right] \leq \frac{2}{T\alpha}  \mathbb E  [\Delta_f] 
 + \frac{2}{T}\frac{\beta_1 D_1}{1-\beta_1} 
  + \frac{2 D_2}{T}  + \frac{3 D_3}{T} + \frac{L}{T\alpha} \sum_{t=1}^T\mathbb E\left[\| Z_{t+1}-  Z_{t}\|^2 \right]\, ,
 \end{align}
where $\Delta_f : = \mathbb E [f( Z_{1})] - \mathbb E [f( Z_{T+1})]$ $D_1, D_2$ and $D_3$ are three terms, defined in Appendix~\ref{app: proof_thm_adm}, and can be tightly bounded from above. 
We first bound $D_3$ using the following quantities of interest:
 
 \begin{align}\notag
\sum_{t=1}^T\left\|  Z_{t} -  \overline X_{t}\right\|^2 \leq T \left( \frac{\beta_1}{1-\beta_1}\right)^2 \alpha^2 d \frac{G_{\infty}^2}{\epsilon} \text{\ \ and\ \ } 
\sum_{t=1}^T\frac{1}{N}\sum_{i=1}^N\left\| {  x_{t,i} -   \overline X_{t}}  \right\|^2 \leq T \alpha^2 \left (\frac{1}{1-\lambda} \right)^2 d G_{\infty}^2 \frac{1}{\epsilon} \,.
 \end{align}
 where $\lambda = \max (|\lambda_2|,|\lambda_N|)$ and recall that $\lambda_i$ is $i$-th largest eigenvalue of $W$.

Then, bounding $D_1$ and $D_2$ give rise to the terms related to $    \mathbb E \left [    \sum_{t=1}^T  \| ( - \hat V_{t-2} + \hat V_{t-1}) \|_{abs}   \right]$.

\vspace{0.1in}

\textsl{Step 3: Bounding the drift term variance.}\hspace{0.01in} An important term that needs upper bounding in our proof is the variance of the gradients multiplied (element-wise) by the adaptive learning rate, $ \mathbb E\left[ \left\| \frac{1}{N} \sum_{i=1}^N \frac{g_{t,i}}{\sqrt{u_{t,i}}} \right\|^2 \right]  \leq   \mathbb E [\| \Gamma_{u}^f\|^2 ] + \frac{d}{N}   \frac{ \sigma^2 }{\epsilon}$,
 where $ \Gamma_{u}^f := 1/N \sum_{i=1}^N \nabla f_i(x_{t,i})/\sqrt{u_{t,i}} $. We can then transform $\mathbb E [\| \Gamma_{u}^f\|^2 ]$ into $\mathbb E [\|\Gamma_{\overline{U}}^f \|^2]$ by splitting out two error terms, then bounding the error terms as operated for $D_2$ and $D_3$. Then, by plugging it into \eqref{eq: exp_telescope_sketchmain}, we obtain the desired bound in Theorem~\ref{thm: dagm_converge}.

\vspace{0.1in}

\noindent\textbf{Proof of Theorem~\ref{thm: dams_converge}:} 
Recall the bound in \eqref{eq: thm1} of Theorem~\ref{thm: dagm_converge}.
Since Algorithm~\ref{alg: damsgrad} is a special case of Algorithm~\ref{alg: dadaptive}, the remaining of the proof consists of characterizing the growth rate of $\mathbb E [ \sum_{t=1}^{T}   \|    (- \hat V_{t-2} + \hat V_{t-1} ) \|_{abs} ]$.
By construction, $\hat V_t$ is non decreasing, so that
$
\mathbb E [ \sum_{t=1}^{T}   \|    (- \hat V_{t-2} + \hat V_{t-1} ) \|_{abs} ] = \mathbb E [   \sum_{i=1}^N \sum_{j=1}^d    (- [\hat v_{0,i}]_j + [\hat v_{T-1,i}]_j ) ]
$.
We can also prove $|[v_{t,i}]_j| \leq G^2_{\infty}$ using $\|g_{t,i}\|_{\infty} \leq G_{\infty}$.
%Besides, since for all $t,i$, $\|g_{t,i}\|_{\infty} \leq G_{\infty}$ and $v_{t,i}$ is an exponential moving average of $g_{k,i}^2, k=1,2,\cdots,t$, we have $|[v_{t,i}]_j| \leq G^2_{\infty}$ for all $t,i,j$. 
%By construction of $\hat V_t$, we also observe that each element of $\hat V_{t}$ cannot be greater than $G^2_{\infty}$, i.e., $|[\hat v_{t,i}]_j| \leq G^2_{\infty}$ for all $t,i,j$.
Then we have $\mathbb E \left[ \sum_{t=1}^{T}   \|    (- \hat V_{t-2} + \hat V_{t-1} ) \|_{abs} \right] \leq  \sum_{i=1}^N \sum_{j=1}^d  \mathbb E[G_{\infty}^2]=  Nd G_{\infty}^2$.
%\begin{align} \notag
%&\mathbb E \left[ \sum_{t=1}^{T}   \|    (- \hat V_{t-2} + \hat V_{t-1} ) \|_{abs} \right] =\mathbb E \left[   \sum_{i=1}^N \sum_{j=1}^d    (- [\hat v_{0,i}]_j + [\hat v_{T-1,i}]_j ) \right]   \nonumber \\
%\leq & \sum_{i=1}^N \sum_{j=1}^d  \mathbb E[G_{\infty}^2]=  Nd G_{\infty}^2 \, .
%\end{align}
%\begin{align} \notag
%\mathbb E \left[ \sum_{t=1}^{T}   \|    (- \hat V_{t-2} + \hat V_{t-1} ) \|_{abs} \right] \leq  \sum_{i=1}^N \sum_{j=1}^d  \mathbb E[G_{\infty}^2]=  Nd G_{\infty}^2 \, .
%\end{align}
Substituting into \eqref{eq: thm1} yields the desired convergence bound for Algorithm~\ref{alg: damsgrad}.

\subsection{Application to AdaGrad algorithm}\label{sec:adagrad}

In this section, we provide a decentralized version of AdaGrad~\citep{duchi2011adaptive} (optionally with momentum) converted by Algorithm~\ref{alg: dadaptive}, further supporting the usefulness of our decentralization framework. 
The required modification for decentralized AdaGrad is to specify line 4 of Algorithm~\ref{alg: dadaptive} as follows: $ \hat v_{t,i} = \frac{t-1}{t} \hat v_{t-1,i} + \frac{1}{t} g_{t,i}^2 $, which is equivalent to  $\hat v_{t,i} = {\frac{1}{t}\sum_{k=1}^t g_{k,i}^2}$. 
In this section, we call this algorithm decentralized AdaGrad.

\newpage

\begin{algorithm}[t]
	\caption{Decentralized AdaGrad (with N nodes)}
	\label{alg: dadagrad}
	\begin{algorithmic}[1]
		\STATE {\bfseries Input:} learning rate $\alpha$, initial point $x_{1,i} = x_{init}, u_{\frac{1}{2},i} = \hat v_{0,i} = \epsilon \mathbf 1\ (\text{with } \epsilon \geq 0), m_{0,i}=0$, mixing matrix $W$ 
 		\FOR{$t = 1,2,\cdots,T$}
		\STATE \textbf{for all }$i \in [N]$ \textbf{do in parallel}
 		\STATE \quad $g_{t,i}  \leftarrow \nabla f_i(x_{t,i}) + \xi_{t,i}$
 		\STATE \quad $m_{t,i} = \beta_1 m_{t-1,i} + (1-\beta_1) g_{t,i}$ 
 		\STATE \quad $ \hat v_{t,i} = \frac{t-1}{t} \hat v_{t-1,i} + \frac{1}{t} g_{t,i}^2 $
 	%	\STATE $\hat v_{t,i} = \max (\hat v_{t-1,i}, v_{t,i} )$
 		\STATE \quad $x_{t+\frac{1}{2},i} = \sum_{j=1}^N W_{ij}x_{t,j}$
 		\STATE \quad $\tilde u_{t,i} = \sum_{j=1}^N W_{ij}\tilde u_{t-\frac{1}{2},j}$
 	    \STATE \quad $u_{t,i} = \max(\tilde u_{t,i}, \epsilon)$
 		\STATE \quad $x_{t+1,i} = x_{t+\frac{1}{2},i} - \alpha \frac{m_{t,i}}{\sqrt{u_{t,i}}}$
 		\STATE \quad $\tilde u_{t+\frac{1}{2},i} = \tilde u_{t,i} - \hat v_{t-1,i} + \hat v_{t,i}$
 		\ENDFOR
 	\end{algorithmic}
 \end{algorithm}

The pseudo code of the algorithm is shown in Algorithm~\ref{alg: dadagrad}. 
There are two details in Algorithm~\ref{alg: dadagrad} worth mentioning.  The first one is that the introduced framework leverages momentum $m_{t,i}$ in updates, while original AdaGrad does not use momentum. 
The momentum can be turned off by setting $\beta_1 = 0$ and the convergence results will still hold. 
The other one is that in Decentralized AdaGrad, we use the average instead of the sum in the term $\hat v_{t,i}$. 
In other words, we write $\hat v_{t,i} = {\frac{1}{t}\sum_{k=1}^t g_{k,i}^2}$. 
This latter point is different from the original AdaGrad which actually uses $\hat v_{t,i} = {\sum_{k=1}^t g_{k,i}^2}$.

The reason is that in the original AdaGrad, a constant stepsize ($\alpha$ independent of $t$ or $T$) is used with  $\hat v_{t,i} = {\sum_{k=1}^t g_{k,i}^2}$. 
This is equivalent to using a well-known decreasing stepsize sequence $\alpha_t = \frac{1}{\sqrt{t}}$ with  $\hat v_{t,i} = {\frac{1}{t}\sum_{k=1}^t g_{k,i}^2}$. 
In our convergence analysis, which can be found below, we use a constant stepsize $\alpha = O(\frac{1}{\sqrt{T}})$ to replace the decreasing stepsize sequence $\alpha_t =  O(\frac{1}{\sqrt{t}})$. 
Such a replacement is popularly used in Stochastic Gradient Descent analysis for the sake of simplicity and to achieve a better convergence rate. 
In addition, it is easy to modify our theoretical framework to include decreasing stepsize sequences such as $\alpha_t =  O(\frac{1}{\sqrt{
t}})$.
The convergence analysis for decentralized AdaGrad is shown in Theorem~\ref{thm: dadagrad_converge}.
\begin{theorem}\label{thm: dadagrad_converge}
Assume A\ref{a:diff}-A\ref{a:matrixW}.
Set $\alpha = \sqrt{N}/\sqrt{Td}$. When $\alpha  \leq \frac{\epsilon^{0.5}}{16L} $, decentralized AdaGrad yields the following regret bound
	  \begin{align}\notag
 \frac{1}{T}\sum_{t=1}^T  \mathbb E \left [\left\|\frac{\nabla f( \overline X_{t})}{\overline U_{t}^{1/4}}\right\|^2  \right]
	  \leq     \frac{C_1' \sqrt{d}}{\sqrt{TN}} D_f'    +  \frac{ C_2'}{T} +  \frac{C_3' N^{1.5}}{T^{1.5}d^{0.5}}   + \frac{\sqrt{N}(1+\log (T))}{T} ( d C_4' + \frac{\sqrt{d} }{T^{0.5} }  C_5')\, ,
	  \end{align}
	where $D_f' := \mathbb E  [f( Z_{1})]  - \min_{z} f(z)]  + \sigma^2$, $C_1' = C_1$, $C_2' = C_2$, $C_3' = C_3$, $C_4' = C_4G_{\infty}^2$ and $C_5' = C_5 G_{\infty}^2 $. $C_1,C_2, C_3, C_4, C_5$ are  defined in Theorem~\ref{thm: dagm_converge} independent of $d$, $T$ and $N$. In addition, the consensus of variables at different nodes is given by $\frac{1}{N}\sum_{i=1}^N\left\| {  x_{t,i} -   \overline X_{t}}  \right\|^2   \leq \frac{N}{T} \left (\frac{1}{1-\lambda} \right)^2  G_{\infty}^2 \frac{1}{\epsilon}$. 
\end{theorem}

\newpage

\section{Numerical Experiments} \label{sec:numerical}

In this section, we conduct some experiments to test the performance of Decentralized AMSGrad, developed in Algorithm~\ref{alg: damsgrad}, on both \emph{homogeneous} data and \emph{heterogeneous} data distribution (i.e., the data generating distribution on different nodes are assumed to be different). 
Comparison with DADAM and the decentralized parallel stochastic gradient descent (D-PSGD) developed in~\cite{lian2017can} are conducted. 
We train a Convolutional Neural Network (CNN) with 3 convolution layers followed by a fully connected layer on MNIST~\citep{lecun1998mnist}.
We set $\epsilon = 10^{-6}$ for both Decentralized AMSGrad and DADAM.
The learning rate is chosen from the grid $[10^{-1}, 10^{-2}, 10^{-3}, 10^{-4}, 10^{-5}, 10^{-6}]$ based on validation accuracy for all algorithms. 
In the following experiments, the graph contains 5 nodes and each node can only communicate with its two adjacent neighbors forming a cycle.
Regarding the mixing matrix $W$, we set $W_{ij} = 1/3$ if nodes $i$ and $j$ are neighbors and $W_{ij} = 0$ otherwise. The implementation was based on the PaddlePaddle deep learning platform. 

\subsection{Effect of heterogeneity}
\begin{figure}[b!]
\begin{center}
\mbox{
    \begin{subfigure}[b]{5.4in}
    \mbox{
      \includegraphics[width=2.7in]{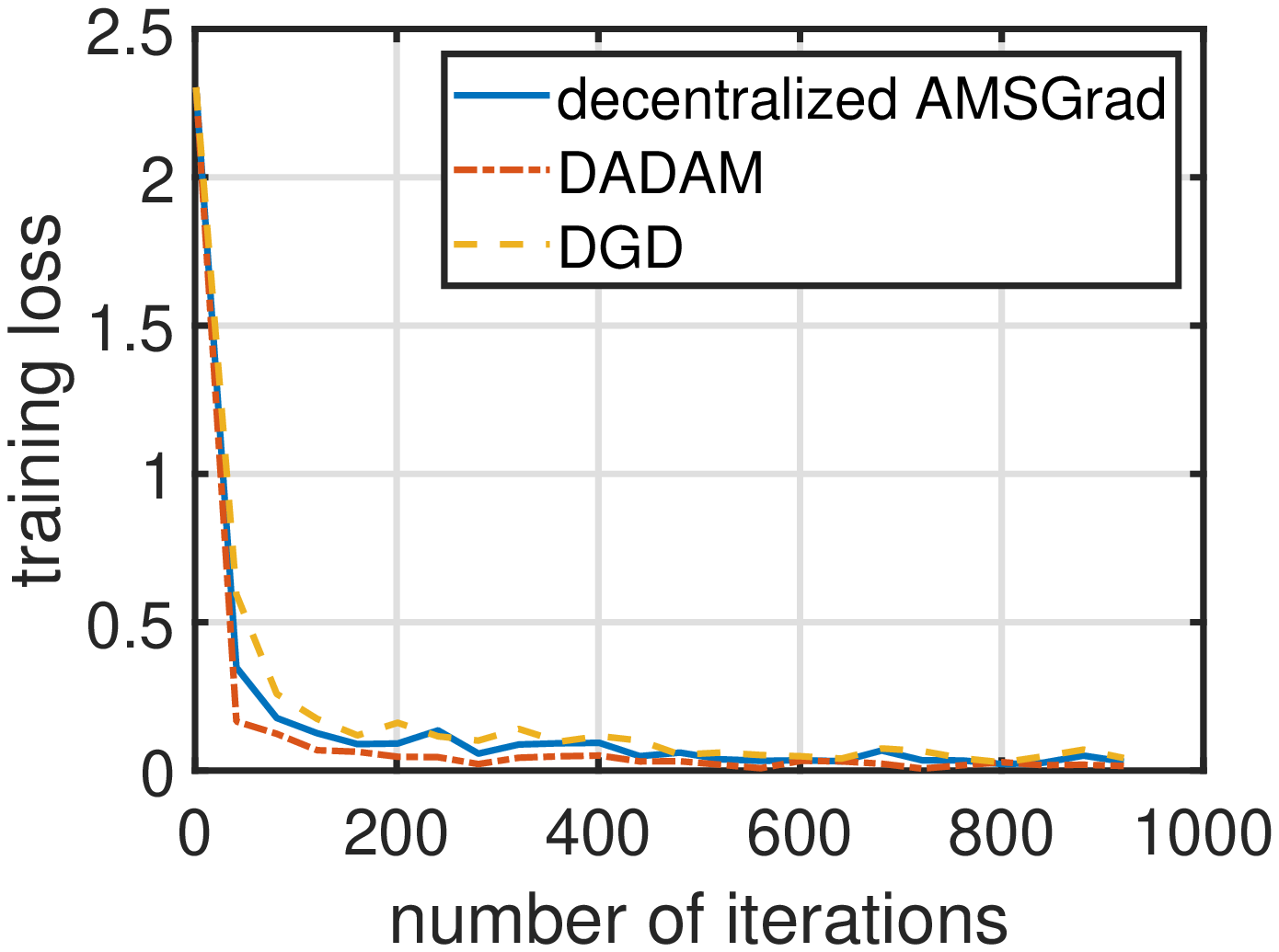}%\hspace{-0.12in}
            \includegraphics[width=2.7in]{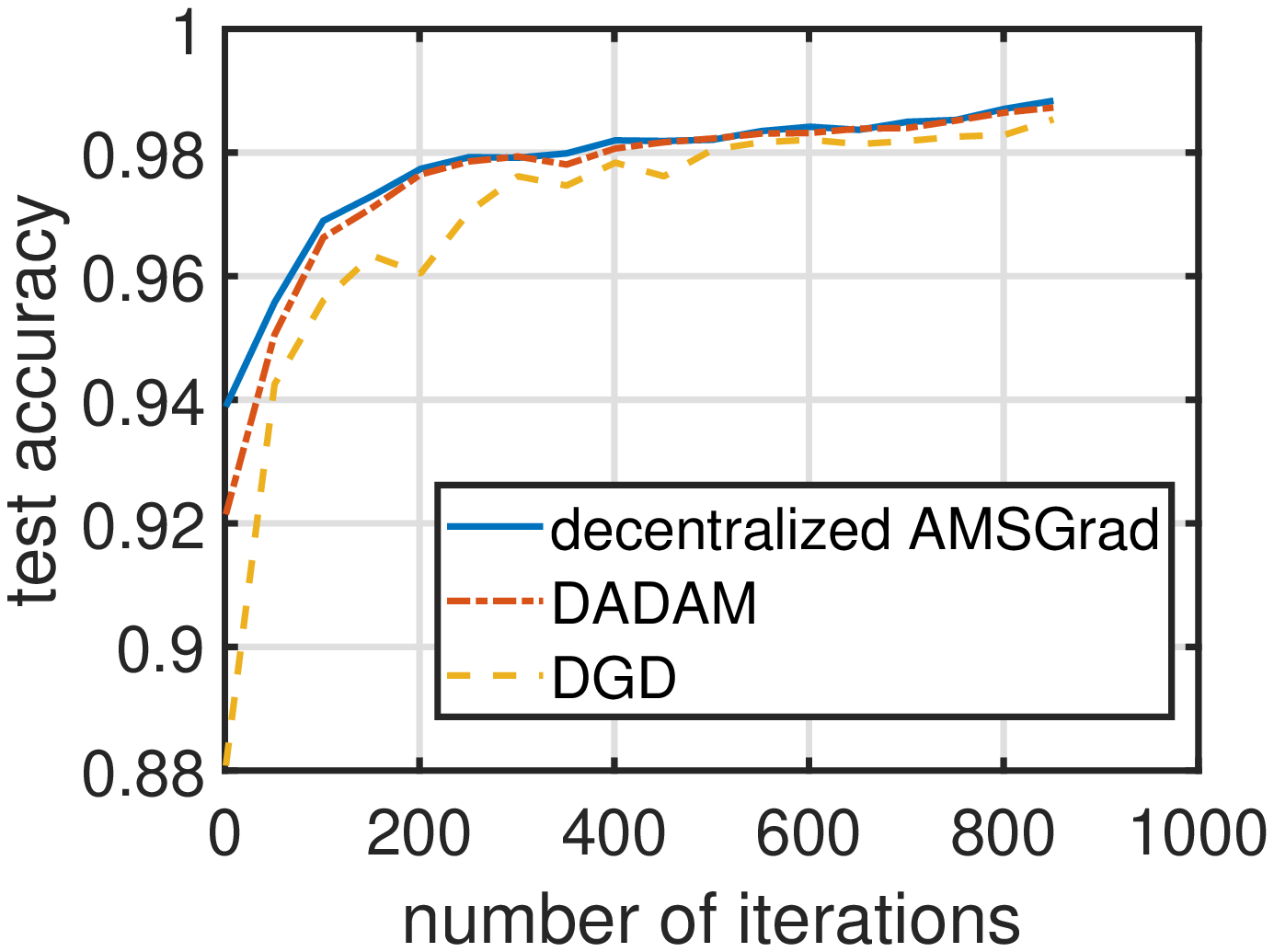}
            }
      \caption{ Homogeneous data} 
    \end{subfigure}
}

\mbox{
    \begin{subfigure}[b]{5.4in}
    \mbox{
      \includegraphics[width=2.7in]{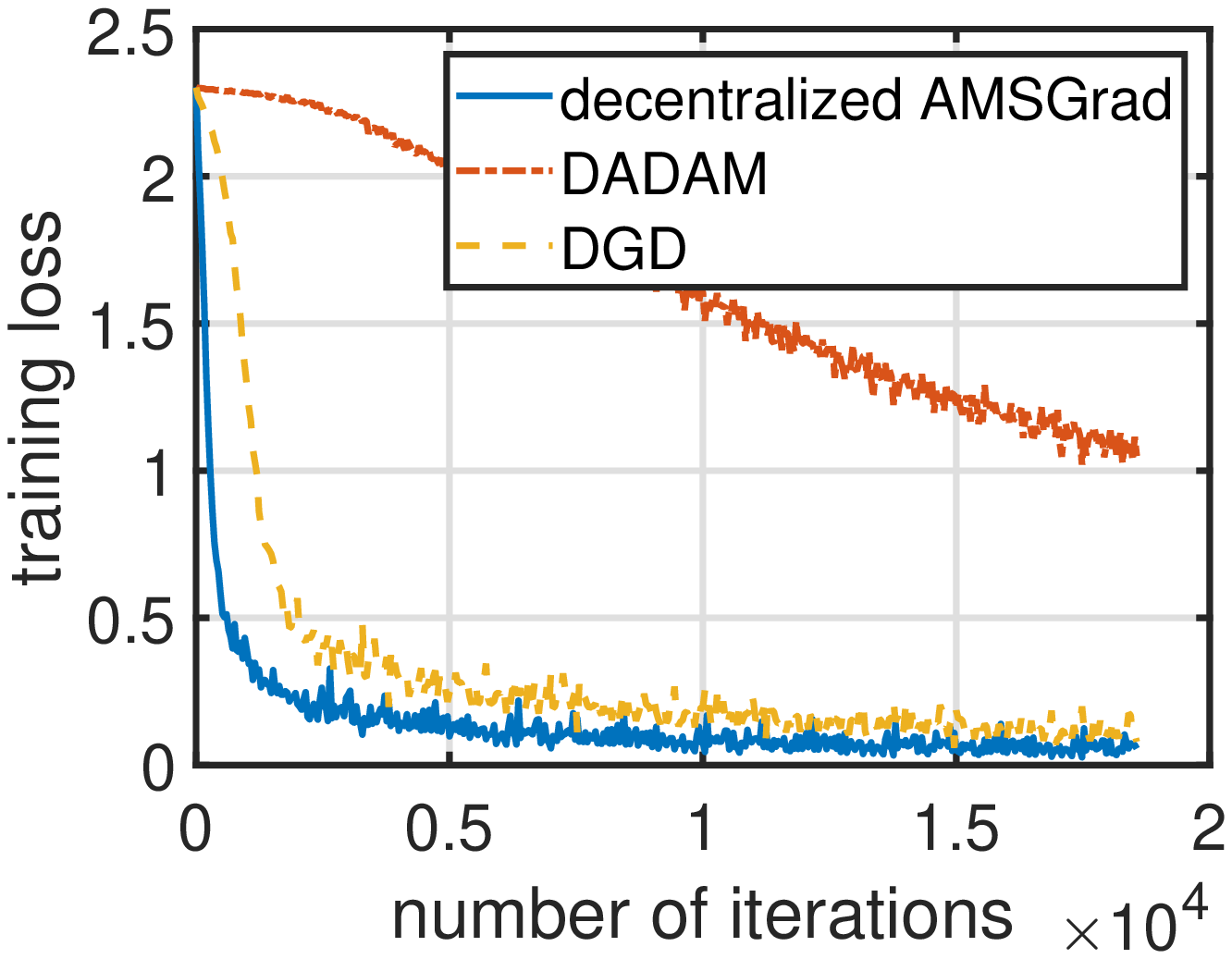}%\hspace{-0.12in}
            \includegraphics[width=2.7in]{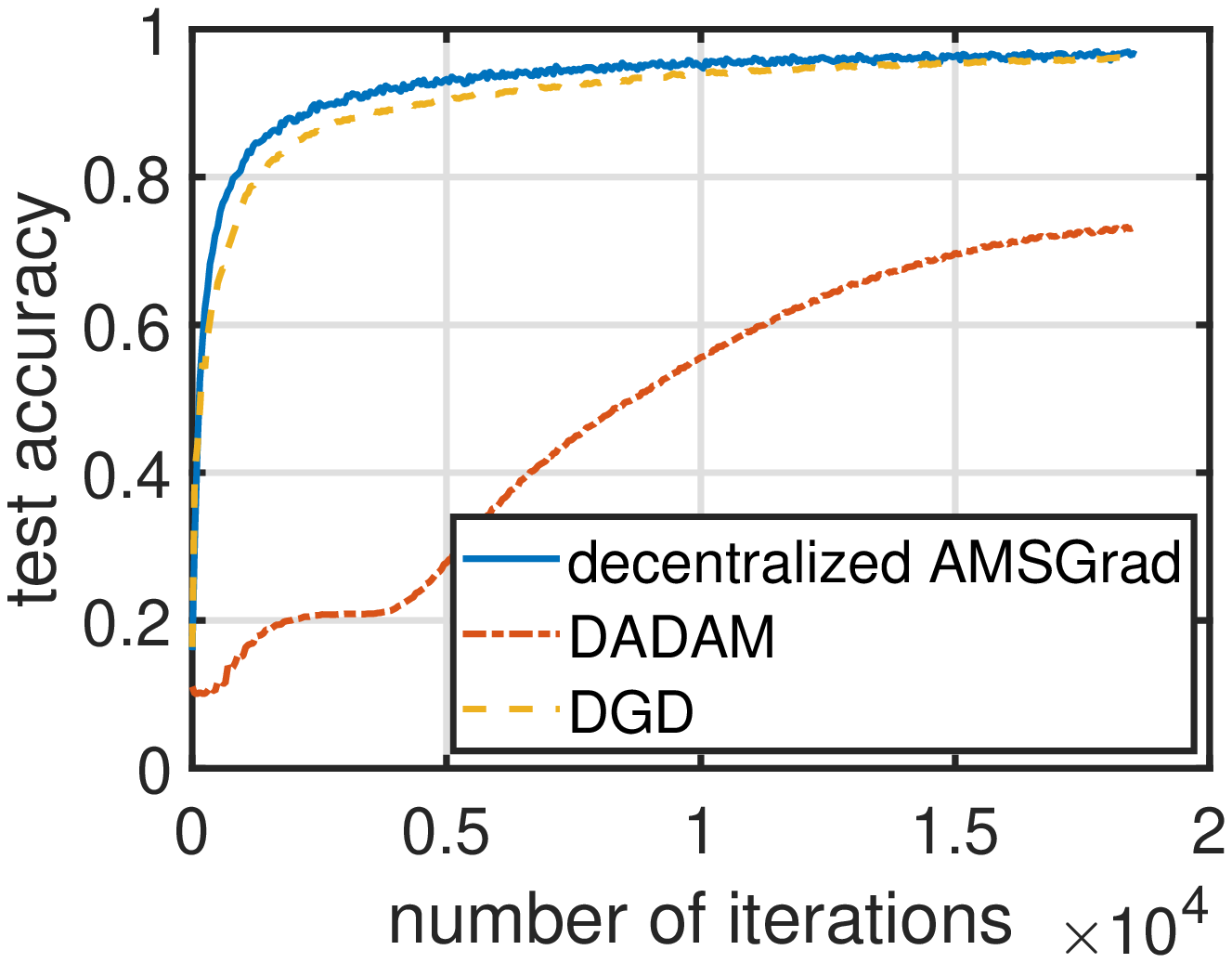}
            }
      \caption{Heterogeneous data}
    \end{subfigure}
}
\end{center}
\caption{Training loss and Testing accuracy for homogeneous and heterogeneous data}
	\label{fig: homo_data}

  \end{figure}

\noindent\textbf{Homogeneous data:}
The whole dataset is shuffled and evenly split into different nodes. {Such a setting is possible when the nodes are in a computer cluster.}
We see, Figure~\ref{fig: homo_data}(a), that decentralized AMSGrad and DADAM perform quite similarly while D-PSGD (labelled as DGD) is much slower both in terms of training loss and test accuracy. 
Though the (possible) non convergence of DADAM, mentioned in this paper, its performance are empirically good on homogeneous data. 
The reason is that the adaptive learning rates tend to be similar on different nodes in presence of homogeneous data distribution. 
We thus compare these algorithms under the heterogeneous regime. 

\vspace{0.1in}
\noindent\textbf{Heterogeneous data:}
Here, each node only contains training data with two labels out of ten. {Such a setting is common when data shuffling is prohibited, such as in  federated learning and other privacy-sensitive scenarios.}
We can see that each algorithm converges significantly slower than with homogeneous data. 
Especially, the performance of DADAM deteriorates significantly. 
Decentralized AMSGrad achieves the best training and testing performance in that setting as observed in Figure~\ref{fig: homo_data}(b). These experiments show that although DADAM is shown to have good performance on homogeneous data in 
 \citet{nazari2019dadam}, heterogeneous data can be detrimental to its performance. On the contrary, decentralized AMSGrad is less impacted by heterogeneous data distribution as a convergent variant, and it enjoys some benefits of adaptive gradient methods.
 
\subsection{Sensitivity to the Learning Rate}

We compare the training loss and testing accuracies of different D-PSGD, DADAM, and our proposed decentralized AMSGrad,  with different stepsizes on \emph{heterogeneous} data distribution. 
We use 5 nodes and the heterogeneous data distribution is created by assigning each node with data of only two labels.
Note that there are no overlapping labels between different nodes.

\begin{figure}[h]

\centering
\mbox{
\begin{subfigure}[b]{2.7in}
		\includegraphics[width=\textwidth]{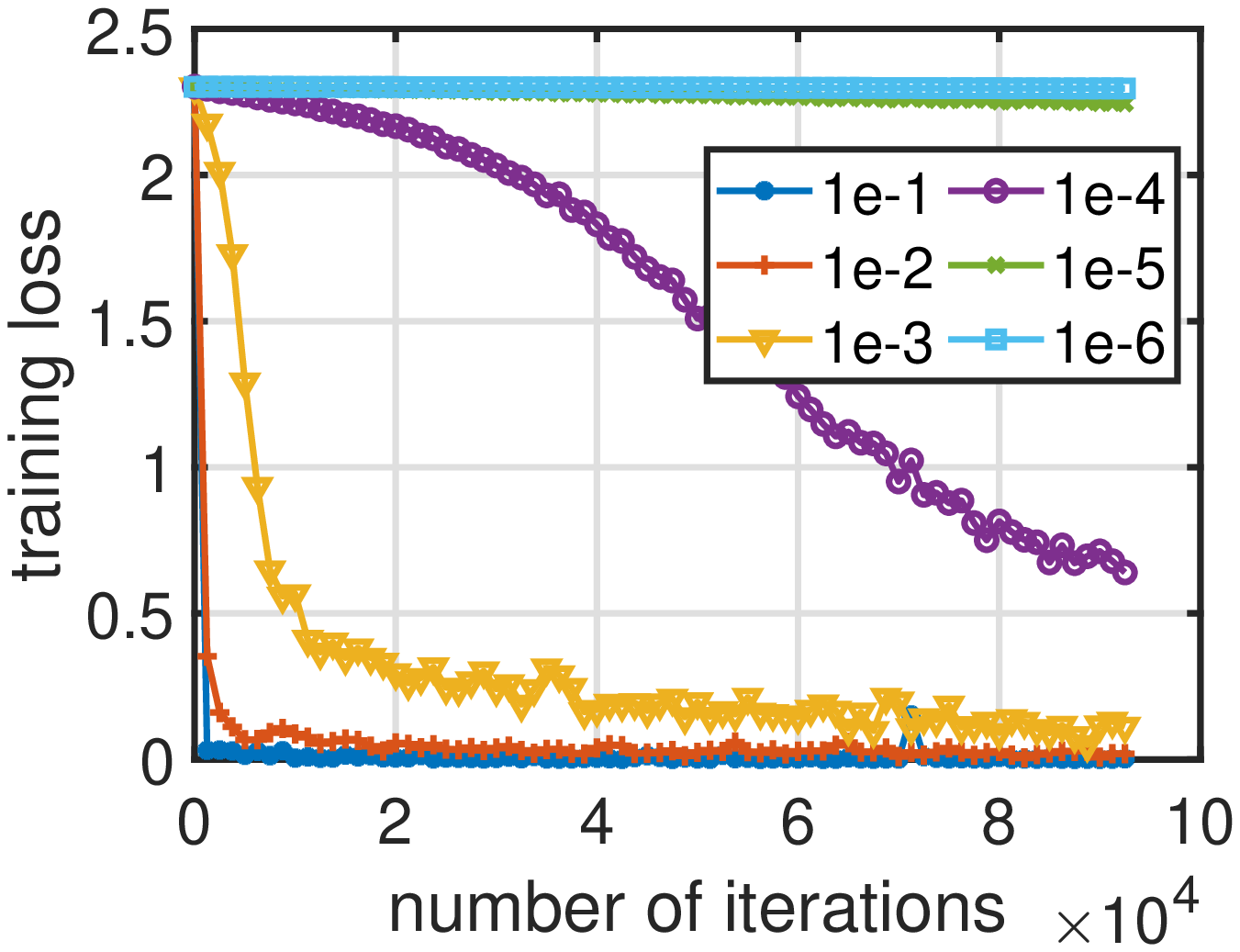}
      \caption{ D-PSGD loss}
    \end{subfigure}
    \begin{subfigure}[b]{2.7in}
		\includegraphics[width=\textwidth]{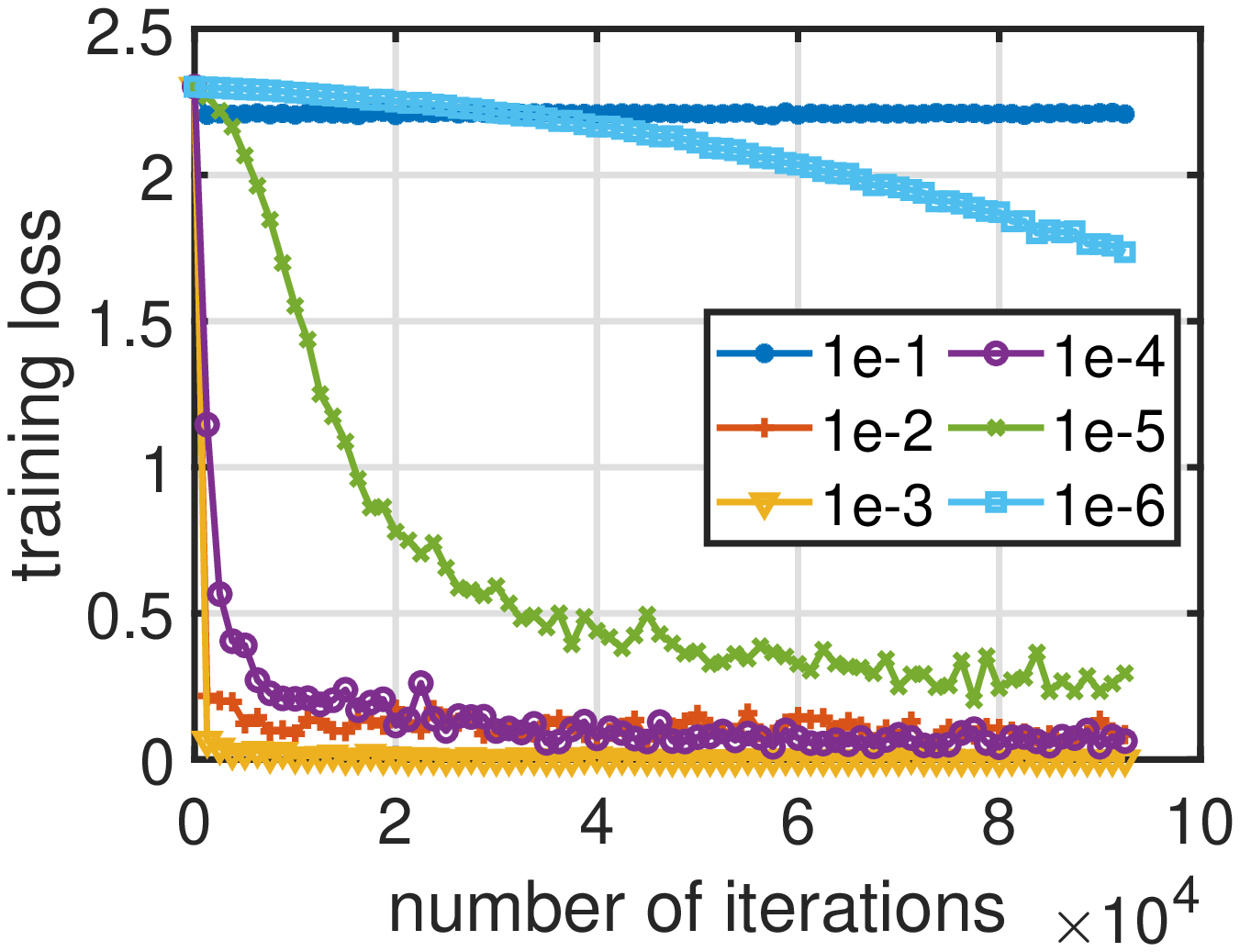}
      \caption{Decentralized AMS loss}
    \end{subfigure}
}

\mbox{    
        \begin{subfigure}[b]{2.7in}
		\includegraphics[width=\textwidth]{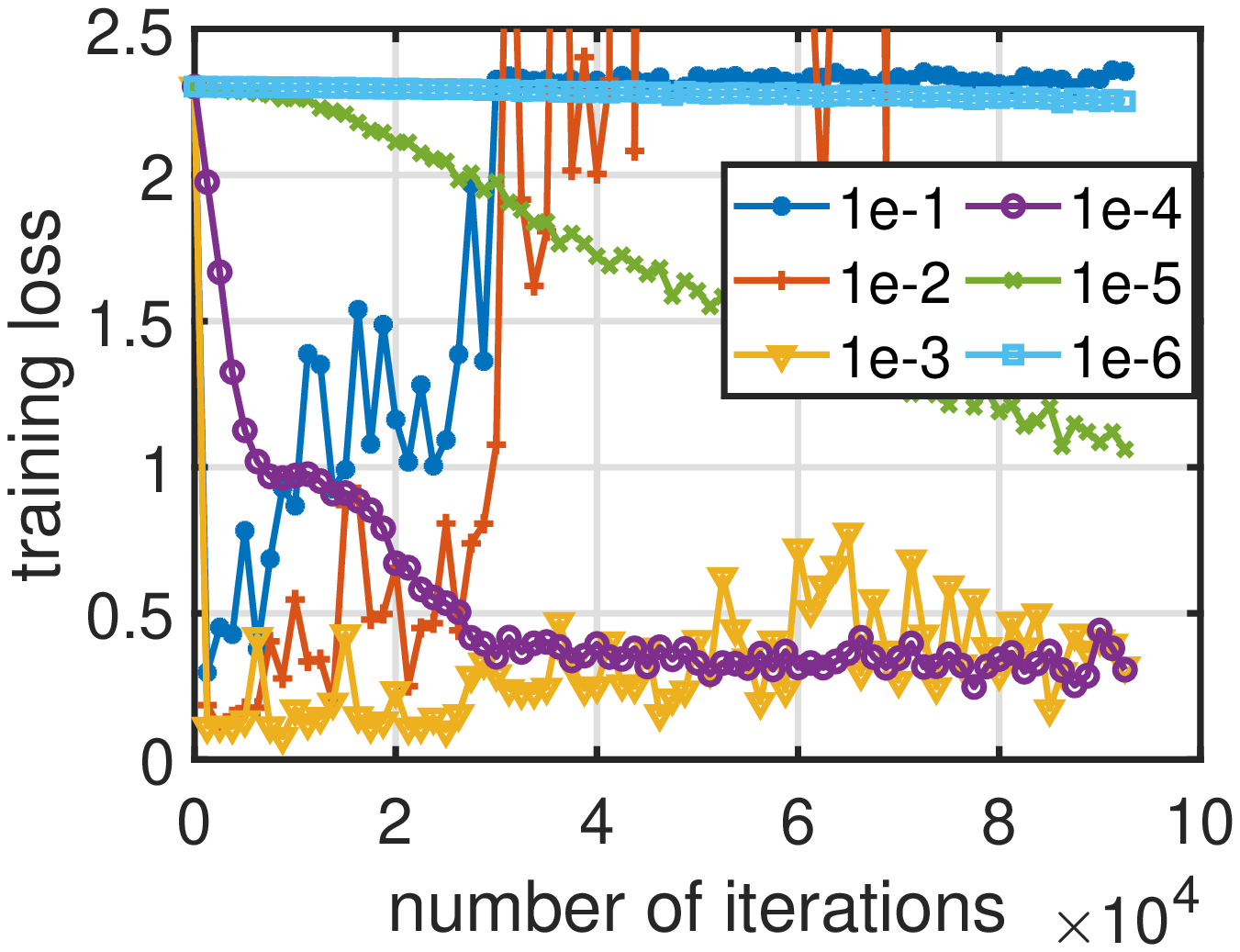}
      \caption{Decentralized Adam loss}
    \end{subfigure}
 
    \begin{subfigure}[b]{2.7in}
		\includegraphics[width=\textwidth]{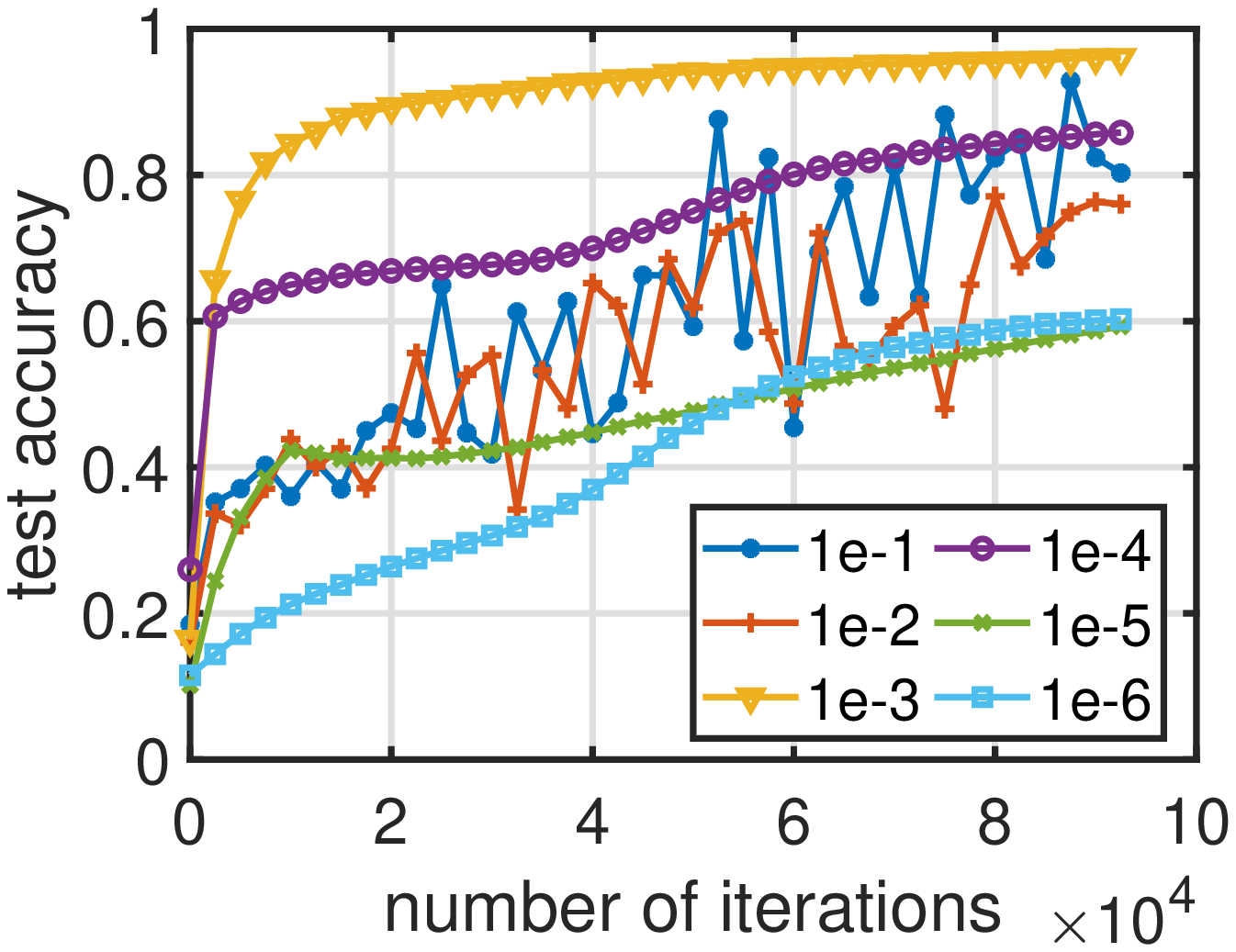}
      \caption{ D-PSGD accuracy}
    \end{subfigure}
}

\mbox{
    \begin{subfigure}[b]{2.7in}
		\includegraphics[width=\textwidth]{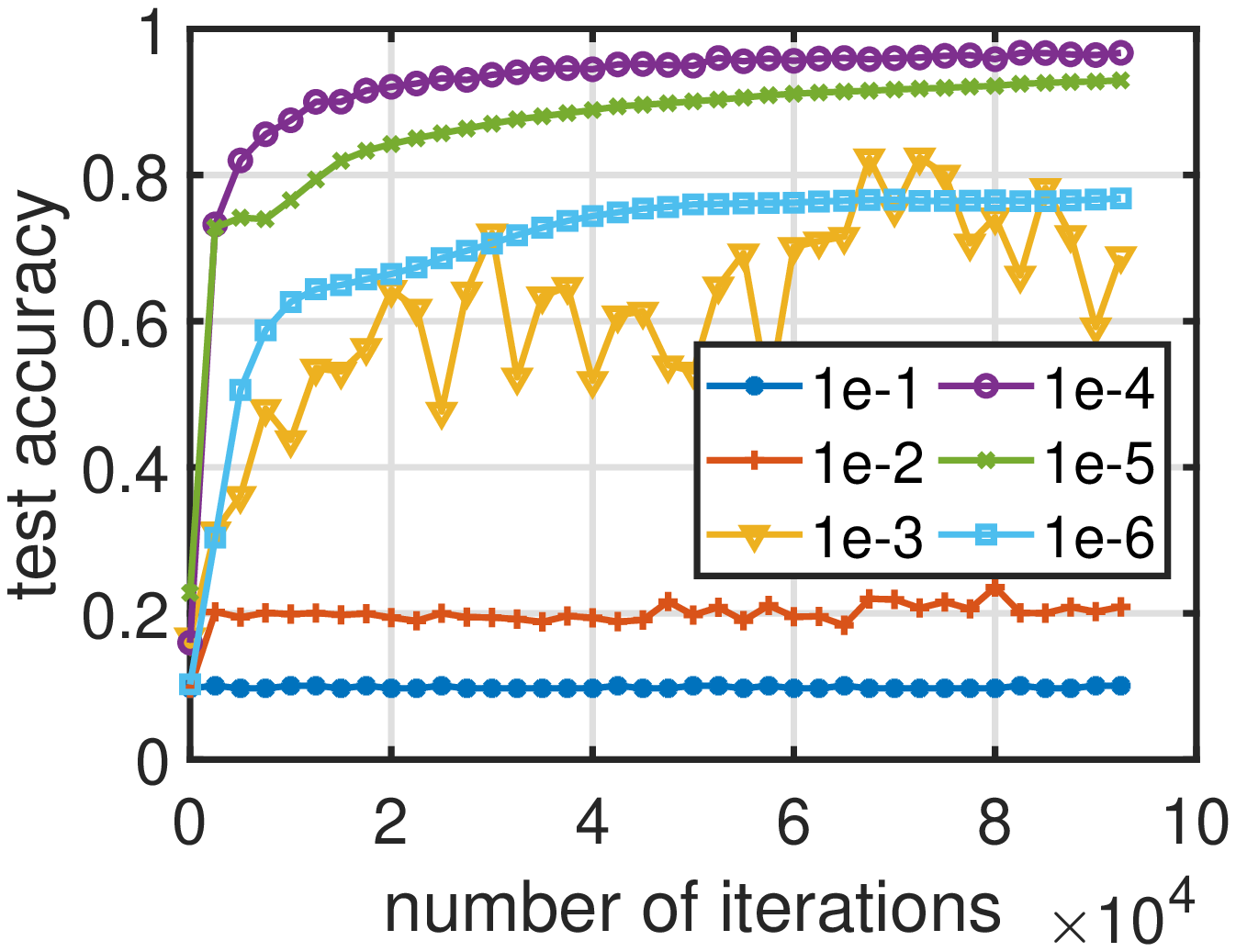}
      \caption{Decentralized AMS accuracy}
    \end{subfigure}
        \begin{subfigure}[b]{2.7in}
		\includegraphics[width=\textwidth]{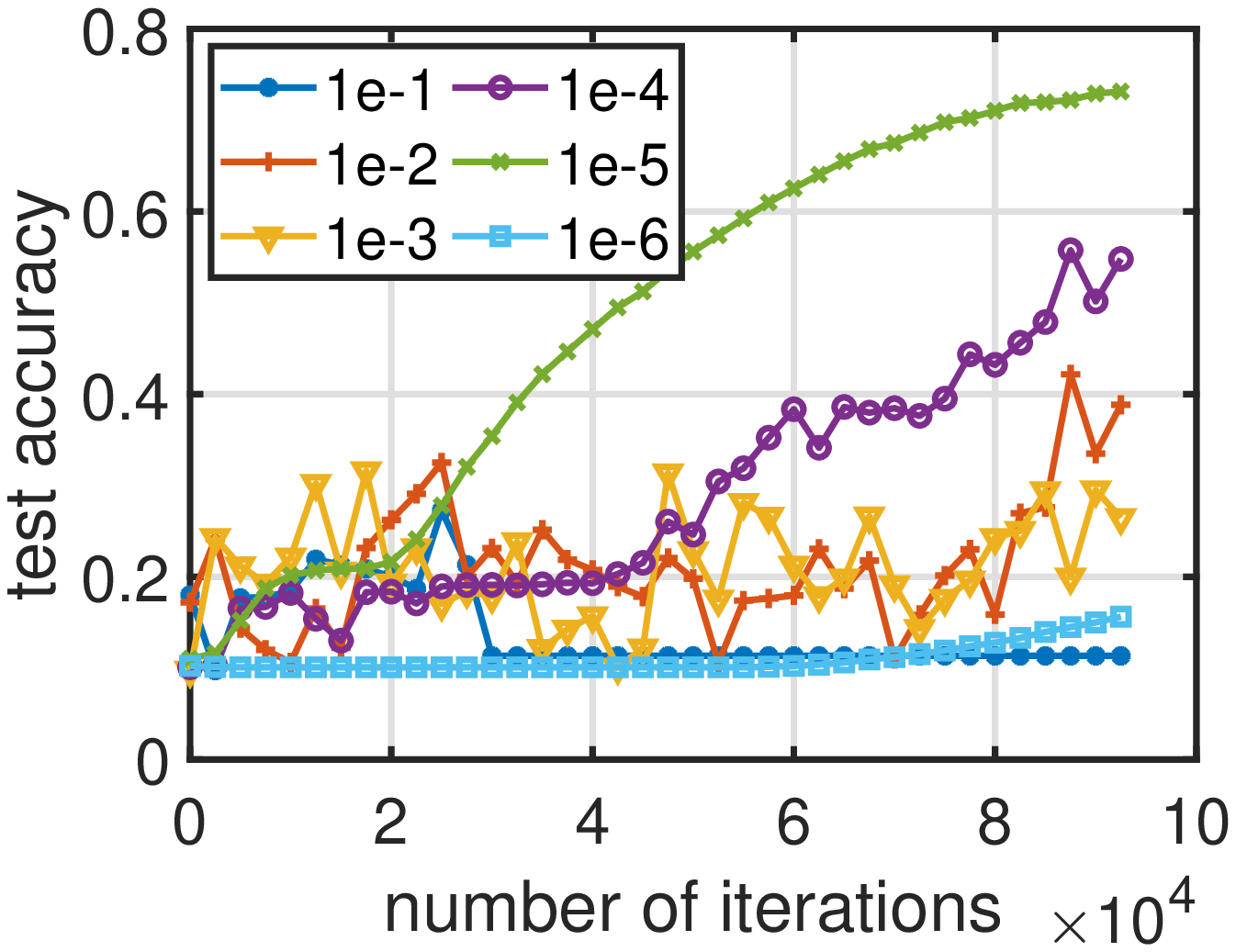}
      \caption{Decentralized Adam accuracy}
    \end{subfigure}
}
	\caption{Training loss and testing accuracy comparison of different stepsizes for various methods}
	\label{fig: stepsize}
	
  \end{figure}

\vspace{0.05in}

We observe Figure\ref{fig: stepsize}(a) and (d) that the stepsize $10^{-3}$ works best for D-PSGD in terms of test accuracy and $10^{-1}$ works best in terms of training loss. 
This difference is caused by the inconsistency among the model parameters on different nodes when the stepsize is large.

\vspace{0.05in}

Figure~\ref{fig: stepsize}(b) and (e) shows the performance of decentralized AMSGrad with different stepsizes.
We see that its best performance is better than the one of D-PSGD and the performance is more stable (the test performance is less sensitive to stepsize tuning).
%Figure~\ref{fig: stepsize}(c) displays the performance of Decentralized Adam algorithm.
As expected, the performance of DADAM is not as good as D-PSGD or decentralized AMSGrad, see Figure~\ref{fig: stepsize}(c) and (f).
Its divergence characteristic, highlighted Section~\ref{sec:divergence}, coupled with the heterogeneity in the data amplify its non-convergence issue in our experiments.  
From the experiments above, we can see the advantages of decentralized AMSGrad in terms of both performance and ease of parameter tuning, and the importance of ensuring the theoretical convergence of any newly proposed methods in the presented~setting.

\section{Conclusion}\label{sec:conclusion}

This paper studies the problem of designing adaptive gradient methods for decentralized training. 
We propose a unifying algorithmic framework that can convert existing adaptive gradient methods to decentralized settings. 
With rigorous convergence analysis, we show that if the original algorithm converges under some minor conditions, the converted algorithm obtained using our proposed framework is guaranteed to converge to stationary points of the regret function. 
By applying our framework to AMSGrad, we propose the first convergent adaptive gradient methods, namely Decentralized AMSGrad. 
We also give an extension to a decentralized variant of AdaGrad for completeness of our converting scheme.
Experiments show that the proposed algorithm achieves better performance than the baselines.

\clearpage 
%%%%%%%%%%%%%%%%%%%%%%%%%%%%%%%%%%%%%%%%%%%%%%%%%%%%
\newpage
\appendix
%\input{Appendix}

%\section{Appendix}

\noindent\textbf{\LARGE Appendix}\\

We provide the proofs for our convergence analysis. After having established several important Lemmas in Section~\ref{app: proof_lemmas}, we provide a proof for  Theorem~\ref{thm: dagm_converge} in Section~\ref{app: proof_thm_adm}.
Section~\ref{app: proof_ams} and Section~\ref{app: proof_adagrad} correspond to the proofs for the extension and application of Theorem~\ref{thm: dagm_converge} to the AMSGrad and AdaGrad algorithms used as prototypes of our general class of decentralized adaptive gradient~methods.

\section{Proof of Auxiliary Lemmas} \label{app: proof_lemmas}

Similarly to~\cite{yan2018unified, chen2018convergence} with SGD (with momentum) and centralized adaptive gradient methods, define the following auxiliary sequence:
 \begin{align}\label{eq: seq_z_sketchapp}
 Z_{t} = \overline X_t + \frac{\beta_1}{1-\beta_1} (\overline X_t - \overline X_{t-1}) \, ,
 \end{align}
with $\overline X_{0} \triangleq \overline X_1$.
Such an auxiliary sequence can help us deal with the bias brought by the momentum and simplifies the convergence analysis. 

 \begin{lemma}\label{lem: z_diff} 
	For the sequence defined in \eqref{eq: seq_z_sketchapp}, we have
	\begin{align}\notag
	Z_{t+1} - Z_t = \alpha \frac{\beta_1}{1-\beta_1}  \frac{1}{N} \sum_{i=1}^N m_{t-1	,i} \odot (\frac{1}{\sqrt{u_{t-1,i}}} - \frac{1}{\sqrt{u_{t,i}}}) 
	- \alpha \frac{1}{N} \sum_{i=1}^N \frac{g_{t,i}}{\sqrt{u_{t,i}}} \, .
	\end{align}
\end{lemma}

\noindent\textbf{Proof:} By update rule of Algorithm~\ref{alg: dadaptive}, we first have
\begin{align}
\overline X_{t+1}  = & \frac{1}{N}\sum_{i=1}^N x_{t+1,i}  
=  \frac{1}{N}\sum_{i=1}^N \left( x_{t+0.5,i} - \alpha \frac{m_{t,i}}{\sqrt{u_{t,i}}}\right) \nonumber   \\
= & \frac{1}{N}\sum_{i=1}^N \left(  \sum_{j=1}^N W_{ij}x_{t,j} - \alpha \frac{m_{t,i}}{\sqrt{u_{t,i}}}\right) \nonumber    \\
\overset{(i)}{=} &  \left(\frac{1}{N} \sum_{j=1}^N x_{t,j} \right) -\frac{1}{N} \sum_{i=1}^N   \alpha \frac{m_{t,i}}{\sqrt{u_{t,i}}}  \nonumber \\
= & \overline X_t - \frac{1}{N} \sum_{i=1}^N   \alpha \frac{m_{t,i}}{\sqrt{u_{t,i}}} \, , \nonumber 
\end{align}
where (i) is due to an interchange of summation and $\sum_{i=1} W_{ij} = 1$.
Then, we have 
\begin{align}
Z_{t+1} - Z_t =& \overline X_{t+1} - \overline X_{t} + \frac{\beta_1}{1-\beta_1} (\overline X_{t+1}- \overline X_t) - \frac{\beta_1}{1-\beta_1} (\overline X_{t+1}- \overline X_t) \nonumber   \\
= &  \frac{1}{1-\beta_1} (\overline X_{t+1}- \overline X_t) - \frac{\beta_1}{1-\beta_1} (\overline X_{t+1}- \overline X_t)  \nonumber  \\
= & \frac{1}{1-\beta_1} \left(- \frac{1}{N} \sum_{i=1}^N   \alpha \frac{m_{t,i}}{\sqrt{u_{t,i}}}\right) - \frac{\beta_1}{1-\beta_1} \left(- \frac{1}{N} \sum_{i=1}^N   \alpha \frac{m_{t-1,i}}{\sqrt{u_{t-1,i}}}\right) \nonumber  \\
= & \frac{1}{1-\beta_1} \left(- \frac{1}{N} \sum_{i=1}^N   \alpha \frac{\beta_1 m_{t-1,i} + (1-\beta_1) g_{t,i}}{\sqrt{u_{t,i}}}\right) - \frac{\beta_1}{1-\beta_1} \left(- \frac{1}{N} \sum_{i=1}^N   \alpha \frac{m_{t-1,i}}{\sqrt{u_{t-1,i}}}\right)  \nonumber  \\
= & \alpha \frac{\beta_1}{1-\beta_1}  \frac{1}{N} \sum_{i=1}^N m_{t-1	,i} \odot (\frac{1}{\sqrt{u_{t-1,i}}} - \frac{1}{\sqrt{u_{t,i}}}) - \alpha \frac{1}{N} \sum_{i=1}^N \frac{g_{t,i}}{\sqrt{u_{t,i}}} \, , \nonumber 
\end{align}
which is the desired result. \hfill $\square$

\begin{lemma}\label{lem: mean_after_max}
	Given  a set of numbers $a_1,\cdots,a_n$ and denote their mean to be $\bar a = \frac{1}{n}\sum_{i=1}^n a_i$. Define $b_i(r) \triangleq = \max(a_i,r)$ and $\bar b (r) =  \frac{1}{n}\sum_{i=1}^n b_i(r)$. For any $r$ and $r'$ with $r' \geq r$ we have 
	\begin{align}\label{eq: r_decrease}
	\sum_{i=1}^n |b_i(r) - \bar b(r)| \geq \sum_{i=1}^n |b_i(r') - \bar b(r')|
	\end{align}
	and when $r \leq \min_{i \in [n]} a_i$, we have
	\begin{align}\label{eq: r_reduce}
	\sum_{i=1}^n |b_i(r) - \bar b(r)| =   \sum_{i=1}^n |a_i - \bar a| \, .
	\end{align}
\end{lemma}

\noindent{\textbf{Proof}:}
Without loss of generality, assume $a_i \leq a_j$ when $i < j$, i.e., $a_i$ is a non-decreasing sequence. Define 
\begin{align}\notag
h(r) = \sum_{i=1}^n |b_i(r) - \bar b(r)| = \sum_{i=1}^n |\max (a_i,r) - \frac{1}{n}\sum_{j=1}^n \max(a_j,r)|\, .
\end{align}
We need to prove that $h$ is a non-increasing function of $r$. 
First,
it is easy to see that $h$ is a continuous function of $r$ with non-differentiable points $r = a_i, i \in [n]$, thus $h$ is a piece-wise linear function.

Next, we will prove that $h(r)$ is  non-increasing in each piece.
Define $ l(r)$ to be the largest index with $a(l(r)) < r$, and $s(r)$ to be the largest index with $a_{s(r)} < \bar b(r)$. Note that we have for $i \leq l(r)$, $b_i(r) = r$ and for $i \leq s(r)$ $b_i(r) - \bar b(r) \leq 0$ since $a_i$ is a non-decreasing sequence. 
Therefore, we have 
\begin{align}\notag
h(r) = \sum_{i=1}^{l(r)} (\bar b(r) - r) + \sum_{i= l(r)+1} ^{s(r) } (\bar b(r) - a_i) + \sum_{i= s(r)+1}^n (a_i - \bar b(r))
\end{align}
and 
\begin{align}\notag
\bar b(r) = \frac{1}{n}\left (l(r) r + \sum_{i=l(r)+1}^n a_i\right) \, .
\end{align}
Taking derivative of the above form, we know the derivative of $h(r)$ at differentiable points is
\begin{align}
h'(r) = &l(r) (\frac{l(r)}{n}-1) + (s(r)-l(r)) \frac{l(r)}{n} - (n-s(r)) \frac{l(r)}{n} \nonumber \\
= & \frac{l(r)}{n} ((l(r) - n) + (s(r) - l(r)) - (n-s(r))) \, . \nonumber
\end{align}

Since we have $s(r) \leq n$ we know $(l(r) - n) + (s(r) - l(r)) - (n-s(r)) \leq 0$ and thus
\begin{align}\notag
h'(r) \leq 0 \, ,
\end{align}
which means $h(r)$ is non-increasing in each piece. Combining with the fact that $h(r)$ is continuous, \eqref{eq: r_decrease} is proven.
When $r \leq a(i)$, we have $b(i) = \max(a_i,r) = r$, for all $r \in [n]$ and $\bar b(r) = \frac{1}{n}\sum_{i=1}^n a_i = \bar a$ which proves \eqref{eq: r_reduce}.
\hfill $\square$

\clearpage

\section{Proof of Theorem~\ref{thm: dagm_converge}}\label{app: proof_thm_adm}

To prove convergence of the algorithm, we first define an auxiliary sequence 
\begin{align}\label{eq: seq_z}
Z_{t} = \overline X_t + \frac{\beta_1}{1-\beta_1} (\overline X_t - \overline X_{t-1}) \, ,
\end{align}
with $\overline X_{0} \triangleq \overline X_1$.
Since $\mathbb E[g_{t,i}] = \nabla f(x_{t,i})$ and $u_{t,i}$ is a function of $G_{1:t-1}$ (which denotes $G_1,G_2,\cdots,G_{t-1}$), we have 
\begin{align}\notag
\mathbb E_{G_t|G_{1:t-1}} \left[\frac{1}{N} \sum_{i=1}^N \frac{g_{t,i}}{\sqrt{u_{t,i}}}\right] =\frac{1}{N} \sum_{i=1}^N \frac{\nabla f_i(x_{t,i})}{\sqrt{u_{t,i}}} \, .
\end{align}

Assuming smoothness (A\ref{a:diff}) we have 
\begin{align}\notag
f( Z_{t+1}) \leq f( Z_{t}) + \langle \nabla f( Z_{t}),  Z_{t+1}-  Z_{t} \rangle + \frac{L}{2}\| Z_{t+1}-  Z_{t}\|^2 \, .
\end{align}

Using Lemma~\ref{lem: z_diff} into the above inequality and take expectation over $G_{t}$ given $G_{1:t-1}$, we have 
\begin{align}
&\mathbb E_{G_t|G_{1:t-1}} [f( Z_{t+1})] \nonumber \\
\leq & f( Z_{t})  - \alpha  \left \langle \nabla f( Z_{t}), \frac{1}{N} \sum_{i=1}^N \frac{\nabla f_i(x_{t,i})}{\sqrt{u_{t,i}}}  \right \rangle + \frac{L}{2} \mathbb E_{G_t|G_{1:t-1}}\left[\| Z_{t+1}-  Z_{t}\|^2 \right]  \nonumber \\
&+ \alpha \frac{\beta_1}{1-\beta_1}  \mathbb E_{G_t|G_{1:t-1}} \left [\left \langle \nabla f( Z_{t}) , \frac{1}{N} \sum_{i=1}^N m_{t-1	,i} \odot (\frac{1}{\sqrt{u_{t-1,i}}} - \frac{1}{\sqrt{u_{t,i}}}) \right \rangle \right] \, . \nonumber
\end{align}

Then take expectation over $G_{1:t-1}$ and rearrange, we have 
\begin{align}\label{eq: exp_lip}
& \alpha  \mathbb E\left[\left \langle \nabla f( Z_{t}), \frac{1}{N} \sum_{i=1}^N \frac{\nabla f_i(x_{t,i})}{\sqrt{u_{t,i}}}  \right \rangle \right] \nonumber \\
\leq & \mathbb E  [f( Z_{t})]  -  \mathbb E [f( Z_{t+1})] + \frac{L}{2} \mathbb E\left[\| Z_{t+1}-  Z_{t}\|^2 \right] \nonumber  \\
&+ \alpha \frac{\beta_1}{1-\beta_1}  \mathbb E \left [\left \langle \nabla f( Z_{t}) , \frac{1}{N} \sum_{i=1}^N m_{t-1	,i} \odot (\frac{1}{\sqrt{u_{t-1,i}}} - \frac{1}{\sqrt{u_{t,i}}}) \right \rangle \right] \, .
\end{align}

In addition, we have 
\begin{align}\label{eq: u_to_u_bar}
&\left \langle \nabla f( Z_{t}), \frac{1}{N} \sum_{i=1}^N \frac{\nabla f_i(x_{t,i})}{\sqrt{u_{t,i}}}  \right \rangle  \nonumber \\
= &  \left \langle \nabla f( Z_{t}), \frac{1}{N} \sum_{i=1}^N \frac{\nabla f_i( x_{t,i})}{\sqrt{\overline U_{t}}}  \right \rangle  +\left \langle \nabla f( Z_{t}), \frac{1}{N} \sum_{i=1}^N \nabla f_i( x_{t,i})\odot \left(\frac{1}{\sqrt{u_{t,i}}} - \frac{1}{\sqrt{\overline U_{t}}}  \right)  \right \rangle 
\end{align}
and the first term on RHS of the equality can be lower bounded as 
\newpage

\begin{align} \label{eq: split_1}
&\left \langle \nabla f( Z_{t}), \frac{1}{N} \sum_{i=1}^N \frac{\nabla f_i( x_{t,i})}{\sqrt{\overline U_{t}}}  \right \rangle \nonumber \\
= &\frac{1}{2} \left\|\frac{\nabla f( Z_{t})}{\overline U_{t}^{1/4}}\right\|^2 + \frac{1}{2}\left\|  \frac{\frac{1}{N}\sum_{i=1}^N \nabla f_i( x_{t,i}) }{\overline U_{t}^{1/4}}  \right\|^2 - \frac{1}{2 }\left\| \frac{\nabla f( Z_{t}) -\frac{1}{N}\sum_{i=1}^N \nabla f_i( x_{t,i})}{\overline U_{t}^{1/4}} \right\|^2 \nonumber \\
\geq & \frac{1}{4} \left\|\frac{\nabla f( \overline X_{t})}{\overline U_{t}^{1/4}}\right\|^2 + \frac{1}{4}\left\|  \frac{ \nabla f( \overline X_{t})}{\overline U_{t}^{1/4}}  \right\|^2 - \frac{1}{2 }\left\| \frac{\nabla f( Z_{t}) -\frac{1}{N}\sum_{i=1}^N \nabla f_i( x_{t,i})}{\overline U_{t}^{1/4}} \right\|^2  \nonumber \\
&- \frac{1}{2} \left\|\frac{\nabla f( Z_{t}) -\nabla f( \overline X_{t})}{\overline U_{t}^{1/4}}\right\|^2 - \frac{1}{2} \left\|  \frac{ \frac{1}{N}\sum_{i=1}^N \nabla f_i( x_{t,i}) -  \nabla f( \overline X_{t})}{\overline U_{t}^{1/4}}  \right\|^2 \nonumber \\
\geq & \frac{1}{2} \left\|\frac{\nabla f( \overline X_{t})}{\overline U_{t}^{1/4}}\right\|^2   - \frac{3}{2} \left\|\frac{\nabla f( Z_{t}) -\nabla f( \overline X_{t})}{\overline U_{t}^{1/4}}\right\|^2 - \frac{3}{2} \left\|  \frac{ \frac{1}{N}\sum_{i=1}^N \nabla f_i( x_{t,i}) -  \nabla f( \overline X_{t})}{\overline U_{t}^{1/4}}  \right\|^2 \, ,
\end{align}
where the inequalities are all due to Cauchy-Schwartz.
Substituting \eqref{eq: split_1} and \eqref{eq: u_to_u_bar} into \eqref{eq: exp_lip}, yields
\begin{align}
\frac{1}{2} \alpha \mathbb E \left [\left\|\frac{\nabla f( \overline X_{t})}{\overline U_{t}^{1/4}}\right\|^2  \right]
\leq & \mathbb E  [f( Z_{t})]  -  \mathbb E [f( Z_{t+1})] + \frac{L}{2} \mathbb E\left[\| Z_{t+1}-  Z_{t}\|^2 \right] \nonumber  \\
&+ \alpha \frac{\beta_1}{1-\beta_1}  \mathbb E \left [\left \langle \nabla f( Z_{t}) , \frac{1}{N} \sum_{i=1}^N m_{t-1	,i} \odot (\frac{1}{\sqrt{u_{t-1,i}}} - \frac{1}{\sqrt{u_{t,i}}}) \right \rangle \right] \nonumber \\
& - \alpha \mathbb E \left [ \left \langle \nabla f( Z_{t}), \frac{1}{N} \sum_{i=1}^N \nabla f_i( x_{t,i})\odot \left(\frac{1}{\sqrt{u_{t,i}}} - \frac{1}{\sqrt{\overline U_{t}}}  \right)  \right \rangle \right] \nonumber \\
& + \frac{3}{2} \alpha \mathbb E \left [ \left\|  \frac{ \frac{1}{N}\sum_{i=1}^N \nabla f_i( x_{t,i}) -  \nabla f( \overline X_{t})}{\overline U_{t}^{1/4}}  \right\|^2 + \left\|\frac{\nabla f( Z_{t}) -\nabla f( \overline X_{t})}{\overline U_{t}^{1/4}}\right\|^2 \right] \, . \nonumber
\end{align}

Then sum over the above inequality from $t= 1$ to $T$ and divide both sides by $T\alpha/2$, we have
\begin{align}\label{eq: exp_telescope}
 \frac{1}{T}\sum_{t=1}^T  \mathbb E \left [\left\|\frac{\nabla f( \overline X_{t})}{\overline U_{t}^{1/4}}\right\|^2  \right]
\leq & \frac{2}{T\alpha} ( \mathbb E  [f( Z_{1})]  -  \mathbb E [f( Z_{T+1})]) + \frac{L}{T\alpha} \sum_{t=1}^T\mathbb E\left[\| Z_{t+1}-  Z_{t}\|^2 \right] \nonumber   \\
&+ \frac{2}{T}\frac{\beta_1}{1-\beta_1} \underbrace{\sum_{t=1}^T   \mathbb E \left [\left \langle \nabla f( Z_{t}) , \frac{1}{N} \sum_{i=1}^N m_{t-1	,i} \odot (\frac{1}{\sqrt{u_{t-1,i}}} - \frac{1}{\sqrt{u_{t,i}}}) \right \rangle \right]}_{D_1} \nonumber  \\
& + \frac{2}{T} \underbrace{\sum_{t=1}^T \mathbb E \left [ \left \langle \nabla f( Z_{t}), \frac{1}{N} \sum_{i=1}^N \nabla f_i( x_{t,i})\odot \left( \frac{1}{\sqrt{\overline U_{t}}} -\frac{1}{\sqrt{u_{t,i}}}  \right)  \right \rangle \right] }_{D_2} \nonumber  \\
& + \frac{3}{T} \underbrace{\sum_{t=1}^T \mathbb E \left [ \left\|  \frac{ \frac{1}{N}\sum_{i=1}^N \nabla f_i( x_{t,i}) -  \nabla f( \overline X_{t})}{\overline U_{t}^{1/4}}  \right\|^2 + \left\|\frac{\nabla f( Z_{t}) -\nabla f( \overline X_{t})}{\overline U_{t}^{1/4}}\right\|^2 \right]}_{D_3} \, . 
\end{align}

Next we need to upper bound all the terms on RHS of the above inequality to obtain the convergence rate.
For the terms composing $D_3$ in \eqref{eq: exp_telescope}, we can upper bound them by
\begin{align}\notag
\left\| \frac{\nabla f( Z_{t}) -  \nabla f( \overline X_{t})}{\overline U_{t}^{1/4}}\right\|^2 & \leq \frac{1}{\min_{j \in [d]}[\overline U_{t}^{1/2}]_j}\left\| \nabla f( Z_{t}) -  \nabla f( \overline X_{t})\right\|^2  \\
& \leq   L \frac{1}{\min_{j \in [d]}[\overline U_{t}^{1/2}]_j} \underbrace{\left\|  Z_{t} -  \overline X_{t}\right\|^2}_{D_4} 
\end{align}
and 
\begin{align}\label{eq: T_3_bound_first}
\left\| \frac{\frac{1}{N}\sum_{i=1}^N \nabla f_i( x_{t,i}) -  \nabla f( \overline X_{t})}{\overline U_{t}^{1/4}}  \right\|^2 
\leq & \frac{1}{\min_{j \in [d]}[\overline U_{t}^{1/2}]_j}  \frac{1}{N} \sum_{i=1}^N\left\| { \nabla f_i( x_{t,i}) -  \nabla f( \overline X_{t})}  \right\|^2 \nonumber \\
\leq & L  \frac{1}{\min_{j \in [d]}[\overline U_{t}^{1/2}]_j}  \frac{1}{N} \underbrace{\sum_{i=1}^N\left\| {  x_{t,i} -   \overline X_{t}}  \right\|^2}_{D_5} \, ,
\end{align}
using Jensen's inequality, Lipschitz continuity of $f_i$, and the fact that $f = \frac{1}{N}\sum_{i=1}^N {f_i}$. 
Next we need to bound $D_4$ and $D_5$.
% Before we proceed into bounding $D_5$, we need some preparations. 
Recall the update rule of $X_t$, we have
\begin{align} \label{eq: update_X}
X_t = X_{t-1} W - \alpha  \frac{M_{t-1}}{\sqrt{U_{t-1}}} = X_{1} W^{t-1} -\alpha \sum_{k=0}^{t-2} \frac{M_{t-k-1}}{\sqrt{U_{t-k-1}}}  W^{k}  \, ,
\end{align}
where we define $W^0 = \mathbf I$.
Since $W$ is a symmetric matrix, we can decompose it as $W = Q \Lambda Q^T$ where $Q$ is a orthonormal matrix and $\Lambda$ is a diagonal matrix whose diagonal elements correspond to eigenvalues of $W$ in an descending order, i.e., $\Lambda_{ii} = \lambda_i$ with $\lambda_i$ being $i$-th largest eigenvalue of $W$. In addition, because $W$ is a doubly stochastic matrix, we know $\lambda_{1} = 1$ and $q_1 = \frac{\mathbf 1_N}{\sqrt{N}}$.
With eigen-decomposition of $W$, we can rewrite $D_5$ as 
\begin{align}\label{eq: t2_matrix}
\sum_{i=1}^N\left\| {  x_{t,i} -   \overline X_{t}}  \right\|^2 =  \|X_t - \overline X_t \mathbf 1^T_N\|_F^2 =  \|X_tQ Q^T -  X_t \frac{1}{N} \mathbf 1_N \mathbf 1^T_N\|_F^2  = \sum_{l=2}^N \|X_t q_l\|^2  \, .
\end{align}
In addition, we can rewrite \eqref{eq: update_X} as 
\begin{align}\label{eq: update_x_decom}
X_t = X_{1} W^{t-1} -\alpha \sum_{k=0}^{t-2} \frac{M_{t-k-1}}{\sqrt{U_{t-k-1}}}  W^{k}   =  X_{1}  -\alpha \sum_{k=0}^{t-2} \frac{M_{t-k-1}}{\sqrt{U_{t-k-1}}} Q\Lambda^{k} Q^T  \, ,
\end{align}
where the last equality is because $x_{1,i} = x_{1,j}$, for all $i,j $ and thus $X_1 W = X_1$.
Then we have when $l > 1$,
\begin{align}\label{eq: x_ql}
X_t q_l = (X_{1}  -\alpha \sum_{k=0}^{t-2} \frac{M_{t-k-1}}{\sqrt{U_{t-k-1}}} Q\Lambda^{k} Q^T ) q_l =  -\alpha \sum_{k=0}^{t-2} \frac{M_{t-k-1}}{\sqrt{U_{t-k-1}}} q_l \lambda_l^{k} \, ,
\end{align}
since $Q$ is orthonormal and $X_1 q_l = x_{1,1} \mathbf 1_N^T q_l = x_{1,1} \sqrt{N} q_1^T q_l  = 0$, for all $l \neq 1$ .

\newpage

Combining \eqref{eq: t2_matrix} and \eqref{eq: x_ql} yields
\begin{align} \notag
D_5 =& 	\sum_{i=1}^N\left\| {  x_{t,i} -   \overline X_{t}}  \right\|^2
= \sum_{l=2}^N \|X_t q_l\|^2 \\\notag
=&  \sum_{l=2}^N \alpha^2 \left \| \sum_{k=0}^{t-2} \frac{M_{t-k-1}}{\sqrt{U_{t-k-1}}} \lambda_{l}^{k}  q_l\right\|^2 \\\label{eq: T_5_bound}
 \leq& \alpha^2 \left (\frac{1}{1-\lambda} \right)^2 Nd G_{\infty}^2 \frac{1}{\epsilon} \,  ,
\end{align}
where the last inequality follows from the fact that $g_{t,i} \leq G_{\infty}$, $\|q_l\| = 1$, and $|\lambda_l| \leq \lambda < 1$.
Now let us turn to $D_4$, it can be rewritten as 
\begin{align}\notag
\left\|  Z_{t} -  \overline X_{t}\right\|^2
=& \left\| \frac{\beta_1}{1-\beta_1} (\overline X_t - \overline X_{t-1}) \right \|^2\\\notag
=&\left( \frac{\beta_1}{1-\beta_1}\right)^2 \alpha^2 \left \|\frac{1}{N}\sum_{i=1}^N \frac{m_{t-1,i}}{\sqrt{u_{t-1,i}}}\right\|^2\\\notag
 \leq& \left( \frac{\beta_1}{1-\beta_1}\right)^2 \alpha^2 d \frac{G_{\infty}^2}{\epsilon}\, .
\end{align}
Now we know both $D_4$ and $D_5$ are in the order of  $\mathcal{O}(\alpha^2)$ and thus $D_3$ is in the order of  $\mathcal{O}(\alpha^2)$.
Next we will bound $D_2$ and $D_1$. Define  $G_1   \triangleq \max_{t \in [T]} \max_{i \in [N]} \|\nabla f_i(x_{t,i})\|_{\infty}$, $G_2   \triangleq \max_{t \in [T]}  \|\nabla f(Z_t)\|_{\infty}$, $G_3  \triangleq \max_{t \in [T]} \max_{i \in [N]} \|g_{t,i}\|_{\infty}$ and $G_{\infty} = \max(G_1,G_2,G_3)$.
Then we have 
\begin{align}\label{eq:T_2_bound}
D_2 =& \sum_{t=1}^T \mathbb E \left [ \left \langle \nabla f( Z_{t}), \frac{1}{N} \sum_{i=1}^N \nabla f_i( x_{t,i})\odot \left( \frac{1}{\sqrt{\overline U_{t}}} -\frac{1}{\sqrt{u_{t,i}}}  \right)  \right \rangle \right] \nonumber \\
\leq & \sum_{t=1}^T \mathbb E \left [  G_{\infty}^2  \frac{1}{N} \sum_{i=1}^N \sum_{j=1}^d \left| \frac{1}{\sqrt{[\overline U_{t}]_j}} -\frac{1}{\sqrt{[u_{t,i}]_{j}}}  \right| \right] \nonumber \\
= & \sum_{t=1}^T \mathbb E \left [  G_{\infty}^2  \frac{1}{N} \sum_{i=1}^N \sum_{j=1}^d \left| \frac{1}{\sqrt{[\overline U_{t}]_j}} -\frac{1}{\sqrt{[u_{t,i}]_{j}}}  \right| \frac{\sqrt{[\overline U_{t}]_j} + \sqrt{[u_{t,i}]_{j}} }{\sqrt{[\overline U_{t}]_j} + \sqrt{[u_{t,i}]_{j}}} \right] \nonumber \\
= & \sum_{t=1}^T \mathbb E \left [  G_{\infty}^2  \frac{1}{N} \sum_{i=1}^N \sum_{j=1}^d \left| \frac{[\overline U_{t}]_j - [u_{t,i}]_{j} }{{[\overline U_{t}]_j}\sqrt{[u_{t,i}]_{j}} + \sqrt{[\overline U_{t}]_j}{[u_{t,i}]_{j}}}  \right| \right] \nonumber \\
\leq &   \mathbb E \bigg [ \underbrace{ \sum_{t=1}^T  G_{\infty}^2  \frac{1}{N} \sum_{i=1}^N \sum_{j=1}^d \left| \frac{[\overline U_{t}]_j - [u_{t,i}]_{j} }{2 \epsilon^{1.5}}  \right| }_{D_6} \bigg ] \, , 
\end{align}
where the last inequality is due to $[u_{t,i}]_j \geq \epsilon$, for all $t,i,j$.

\newpage

To simplify notations, define $\|A\|_{abs} = \sum_{i,j} |A_{ij}|$ to be the entry-wise $L_1$ norm of a matrix $A$, then we obtain
\begin{align}\notag
D_6 \leq&  \frac{G_{\infty}^2}{N}\sum_{t=1}^T \frac{1}{2\epsilon^{1.5}} \|\overline U_{t} \mathbf{1}^T - U_t \|_{abs} \\\notag
\leq & \frac{G_{\infty}^2}{N}\sum_{t=1}^T \frac{1}{2\epsilon^{1.5}} \|\overline{ \tilde U}_{t} \mathbf{1}^T - \tilde{U}_t \|_{abs} \nonumber  \\
= & \frac{G_{\infty}^2}{N}\sum_{t=1}^T \frac{1}{2\epsilon^{1.5}} \|  \tilde U_{t} \frac{1}{N} \mathbf{1}_N\mathbf{1}_N^T - \tilde U_t Q Q^T \|_{abs}\nonumber   \\
= & \frac{G_{\infty}^2}{N}\sum_{t=1}^T \frac{1}{2\epsilon^{1.5}}  \| - \sum_{l=2}^N   \tilde U_t q_l q_l^T \|_{abs}  \, , \nonumber 
\end{align}
where the second inequality is due to Lemma~\ref{lem: mean_after_max}, introduced Section~\ref{app: proof_lemmas}, and the fact that $U_t = \max(\tilde U_t,\epsilon)$ (element-wise max operator).
Recall from update rule of $U_t$, by defining $\hat V_{-1} \triangleq \hat V_{0}$ and $U_0 \triangleq U_{1/2}$, we have for all $t \geq 0$, $\tilde U_{t+1} = (\tilde U_t  - \hat V_{t-1} + \hat V_{t})W$.
Thus, we obtain
\begin{equation}\notag
\tilde U_{t} = \tilde U_0 W^t + \sum_{k=1}^t (- \hat V_{t-1-k} + \hat V_{t-k} ) W^k =  \tilde U_0 + \sum_{k=1}^t (- \hat V_{t-1-k} + \hat V_{t-k} ) Q \Lambda^k Q^T \, .
\end{equation}
Then we further obtain when $l \neq 1$,
\begin{equation}\notag
\tilde U_t q_l = (\tilde U_0 + \sum_{k=1}^t (- \hat V_{t-1-k} + \hat V_{t-k} ) Q \Lambda^k Q^T) q_l =   \sum_{k=1}^t (- \hat V_{t-1-k} + \hat V_{t-k} ) q_l \lambda_l^k \, ,
\end{equation}
where the last equality is due to the definition $\tilde U_0 \triangleq U_{1/2} =  \epsilon \mathbf{1_d} \mathbf 1_N^T = \sqrt{N}  \epsilon \mathbf{1_d} \mathbf 1_N^T$ (recall that $q_1 = \frac{1}{\sqrt{N}}\mathbf 1_N^T$) and $q_i^T q_j = 0$ when $i \neq j$.
Note that by definition of $\|\cdot \|_{abs}$, we have for all $A, B, \|A+B\|_{abs} \leq \|A\|_{abs} + \|B\|_{abs} $, then
\begin{align}\label{eq: T_6_bound}
D_6 \leq & \frac{G_{\infty}^2}{N}\sum_{t=1}^T \frac{1}{2\epsilon^{1.5}}  \| - \sum_{l=2}^N  \tilde U_t q_l q_l^T \|_{abs} \nonumber \\
=  & \frac{G_{\infty}^2}{N}\sum_{t=1}^T \frac{1}{2\epsilon^{1.5}}  \| -   \sum_{k=1}^t (- \hat V_{t-1-k} + \hat V_{t-k} )  \sum_{l=2}^N q_l \lambda_l^k q_l^T \|_{abs} \nonumber  \\ 
% \leq & \frac{G_{\infty}^2}{N}\sum_{t=1}^T \frac{1}{2\epsilon^{1.5}}  \sum_{k=1}^t \|     (- \hat V_{t-1-k} + \hat V_{t-k} )  \sum_{l=2}^N q_l \lambda_l^k q_l^T \|_{abs}  \\
% = &    \frac{G_{\infty}^2}{N}\sum_{t=1}^T \frac{1}{2\epsilon^{1.5}}  \sum_{k=1}^t \sum_{j=1}^d \| \sum_{l=2}^N q_l \lambda_l^k q_l^T     (- \hat V_{t-1-k} + \hat V_{t-k} )^T e_j  \|_{1}  \\
\leq &  \frac{G_{\infty}^2}{N}\sum_{t=1}^T \frac{1}{2\epsilon^{1.5}}  \sum_{k=1}^t  \sum_{j=1}^d \| \sum_{l=2}^N q_l \lambda_l^k q_l^T \|_{1}  \|     (- \hat V_{t-1-k} + \hat V_{t-k} )^T e_j \|_1  \nonumber \\
\leq &  \frac{G_{\infty}^2}{N}\sum_{t=1}^T \frac{1}{2\epsilon^{1.5}}  \sum_{k=1}^t  \sum_{j=1}^d  \sqrt{N}\| \sum_{l=2}^N q_l \lambda_l^k q_l^T \|_{2}  \|     (- \hat V_{t-1-k} + \hat V_{t-k} )^T e_j \|_1  \nonumber \\
\leq  & \frac{G_{\infty}^2}{N}\sum_{t=1}^T \frac{1}{2\epsilon^{1.5}}  \sum_{k=1}^t \sum_{j=1}^d \|    (- \hat V_{t-1-k} + \hat V_{t-k} )^T e_j\|_1 \sqrt{N} \lambda^k \nonumber  \\
=  & \frac{G_{\infty}^2}{N}\sum_{t=1}^T \frac{1}{2\epsilon^{1.5}}  \sum_{k=1}^t  \|    (- \hat V_{t-1-k} + \hat V_{t-k} ) \|_{abs} \sqrt{N} \lambda^k \nonumber \\
% =  & \frac{G_{\infty}^2}{N}\sum_{t=1}^T \frac{1}{2\epsilon^{1.5}}  \sum_{o=0}^{t-1}  \|    (- \hat V_{o-1} + \hat V_{o} ) \|_{abs} \sqrt{N} \lambda^{t-o}  \\
=  & \frac{G_{\infty}^2}{N}\frac{1}{2\epsilon^{1.5}} \sum_{o=0}^{T-1} \sum_{t=o+1}^T     \|    (- \hat V_{o-1} + \hat V_{o} ) \|_{abs} \sqrt{N} \lambda^{t-o} \nonumber  \\ 
\leq & \frac{G_{\infty}^2}{\sqrt{N}}\frac{1}{2\epsilon^{1.5}} \sum_{o=0}^{T-1} \frac{\lambda}{1-\lambda}     \|    (- \hat V_{o-1} + \hat V_{o} ) \|_{abs} \, ,
\end{align}
where $\lambda = \max (|\lambda_2|,|\lambda_N|)$.
Combining \eqref{eq:T_2_bound} and \eqref{eq: T_6_bound}, we have
\begin{equation}\notag
D_2 \leq  \frac{G_{\infty}^2}{\sqrt{N}}\frac{1}{2\epsilon^{1.5}} \frac{\lambda}{1-\lambda}   \mathbb E \left[ \sum_{o=0}^{T-1}     \|    (- \hat V_{o-1} + \hat V_{o} ) \|_{abs} \right]  \, .
\end{equation}
Now we need to bound $D_1$, we have
\begin{align}\label{eq: T_1}
D_1 = & \sum_{t=1}^T   \mathbb E \left [\left \langle \nabla f( Z_{t}) , \frac{1}{N} \sum_{i=1}^N m_{t-1	,i} \odot (\frac{1}{\sqrt{u_{t-1,i}}} - \frac{1}{\sqrt{u_{t,i}}}) \right \rangle \right] \nonumber \\
\leq & \sum_{t=1}^T   \mathbb E \left [   G_{\infty}^2 \frac{1}{N} \sum_{i=1}^N \sum_{j=1}^d \bigg|\frac{1}{\sqrt{[u_{t-1,i}]_j}} - \frac{1}{\sqrt{[u_{t,i}]_j}}\bigg|   \right] \nonumber \\
= & \sum_{t=1}^T   \mathbb E \left [   G_{\infty}^2 \frac{1}{N} \sum_{i=1}^N \sum_{j=1}^d \left|\left(\frac{1}{\sqrt{[u_{t-1,i}]_j}} - \frac{1}{\sqrt{[u_{t,i}]_j}}\right) \frac{\sqrt{[u_{t,i}]_j}+\sqrt{[u_{t-1,i}]_j}}{\sqrt{[u_{t,i}]_j}+\sqrt{[u_{t-1,i}]_j}}\right|    \right] \nonumber \\
\leq & \sum_{t=1}^T   \mathbb E \left [   G_{\infty}^2 \frac{1}{N} \sum_{i=1}^N \sum_{j=1}^d \left|\frac{1}{2\epsilon^{1.5}}\left({{[u_{t-1,i}]_j}} - {{[u_{t,i}]_j}}\right) \right|    \right] \nonumber \\
\overset{(a)}{\leq} & \sum_{t=1}^T   \mathbb E \left [   G_{\infty}^2 \frac{1}{N} \sum_{i=1}^N \sum_{j=1}^d\frac{1}{2\epsilon^{1.5}} \left|\left({{[\tilde u_{t-1,i}]_j}} - {{[\tilde u_{t,i}]_j}}\right) \right|    \right]  \nonumber \\
= &  G_{\infty}^2 \frac{1}{2\epsilon^{1.5}} \frac{1}{N}   \mathbb E \left [  \sum_{t=1}^T   \|{{\tilde U_{t-1}}} - {{\tilde U_{t}}\|_{abs}}    \right] \, ,  
\end{align}
where $(a)$ is due to $[\tilde u_{t-1,i}]_j = \max ([u_{t-1,i}]_j,\epsilon)$ and the function $\max(\cdot,\epsilon)$ is 1-Lipschitz.
In addition, by update rule of $U_t$, we have 
\begin{align}\label{eq: diff_u_t}
 \sum_{t=1}^T   \|{{\tilde U_{t-1}}} - {{\tilde U_{t}}\|_{abs}} 
= &       \sum_{t=1}^T   \|{{\tilde U_{t-1}}} - (\tilde U_{t-1}  - \hat V_{t-2} + \hat V_{t-1})W \|_{abs}   \nonumber   \\
% = &    \sum_{t=1}^T   \|\tilde U_{t-1}(I-W)  + (- \hat V_{t-2} + \hat V_{t-1})W\|_{abs}      \\
= &   \sum_{t=1}^T   \|\tilde U_{t-1}(QQ^T-Q\Lambda Q^T)  + (- \hat V_{t-2} + \hat V_{t-1})W \|_{abs}  \nonumber \\
= &  \sum_{t=1}^T   \|\tilde U_{t-1}(\sum_{l=2}^N q_l (1-\lambda_l)q_l^T)  + (- \hat V_{t-2} + \hat V_{t-1})W\|_{abs}   \nonumber   \\
\leq &  \sum_{t=1}^T   \| \sum_{k=1}^{t-1} (- \hat V_{t-2-k} + \hat V_{t-1-k} ) \sum_{l=2}^N q_l \lambda_l^k  (1-\lambda_l)q_l^T  \|_{abs} + \sum_{t=1}^T  \| (- \hat V_{t-2} + \hat V_{t-1})W \|_{abs}   \nonumber   \\
\leq &   \sum_{t=1}^T  \left(  \sum_{k=1}^{t-1} \|- \hat V_{t-2-k} + \hat V_{t-1-k}\|_{abs} \sqrt{N}\lambda^k \right)   + \sum_{t=1}^T  \| ( - \hat V_{t-2} + \hat V_{t-1}) \|_{abs} \nonumber  \\
=  &  \sum_{t=1}^T  \left(  \sum_{o=1}^{t-1} \|- \hat V_{o-2} + \hat V_{o-1}\|_{abs} \sqrt{N}\lambda^{t-o} \right)   + \sum_{t=1}^T  \| ( - \hat V_{t-2} + \hat V_{t-1}) \|_{abs}    \nonumber  \\
=  &\sum_{o=1}^{T-1}  \sum_{t=o+1}^T  \left(   \|- \hat V_{o-2} + \hat V_{o-1}\|_{abs} \sqrt{N}\lambda^{t-o} \right)   + \sum_{t=1}^T  \| ( - \hat V_{t-2} + \hat V_{t-1}) \|_{abs}  \nonumber  \\
\leq &\sum_{o=1}^{T-1} \frac{\lambda}{1-\lambda}   \left(   \|- \hat V_{o-2} + \hat V_{o-1}\|_{abs} \sqrt{N}  \right)   + \sum_{t=1}^T  \| ( - \hat V_{t-2} + \hat V_{t-1}) \|_{abs}   \nonumber \\
\leq & \frac{1}{1-\lambda}   \sum_{t=1}^T  \| ( - \hat V_{t-2} + \hat V_{t-1}) \|_{abs}  \sqrt{N}    \, .
\end{align}
Combining \eqref{eq: T_1} and \eqref{eq: diff_u_t} yields
\begin{align}
D_1 \leq G_{\infty}^2 \frac{1}{2\epsilon^{1.5}} \frac{1}{N}   \mathbb E \left [  \frac{1}{1-\lambda}   \sum_{t=1}^T  \| ( - \hat V_{t-2} + \hat V_{t-1}) \|_{abs}  \sqrt{N} \right]\, .
\end{align}

What remains is to bound $\sum_{t=1}^T \mathbb E\left[\| Z_{t+1}-  Z_{t}\|^2 \right]$. By update rule of $Z_t$, we have
\begin{align}\notag
\| Z_{t+1}-  Z_{t}\|^2 
 = &  \left\| \alpha \frac{\beta_1}{1-\beta_1}  \frac{1}{N} \sum_{i=1}^N m_{t-1	,i} \odot (\frac{1}{\sqrt{u_{t-1,i}}} - \frac{1}{\sqrt{u_{t,i}}}) - \alpha \frac{1}{N} \sum_{i=1}^N \frac{g_{t,i}}{\sqrt{u_{t,i}}} \right\|^2 \nonumber \\
\leq & 2 \alpha^2 \left\|  \frac{\beta_1}{1-\beta_1}  \frac{1}{N} \sum_{i=1}^N m_{t-1	,i} \odot (\frac{1}{\sqrt{u_{t-1,i}}} - \frac{1}{\sqrt{u_{t,i}}})\right\|^2 + 2 \alpha^2 \left\| \frac{1}{N} \sum_{i=1}^N \frac{g_{t,i}}{\sqrt{u_{t,i}}} \right\|^2 \nonumber \\
\leq & 2 \alpha^2 \left ( \frac{\beta_1}{1-\beta_1} \right)^2    G_{\infty} ^2 \frac{1}{N} \sum_{i=1}^N \sum_{j=1}^d   \frac{1}{\sqrt{\epsilon}}\left|\frac{1}{\sqrt{[u_{t-1,i}]_j}} - \frac{1}{\sqrt{[u_{t,i}]_j}}\right| + 2 \alpha^2 \left\| \frac{1}{N} \sum_{i=1}^N \frac{g_{t,i}}{\sqrt{u_{t,i}}} \right\|^2 \nonumber \\
\leq & 2 \alpha^2 \left ( \frac{\beta_1}{1-\beta_1} \right)^2    G_{\infty} ^2 \frac{1}{N} \sum_{i=1}^N \sum_{j=1}^d   \frac{1}{\sqrt{\epsilon}}\left|\frac{[u_{t,i}]_j - [u_{t-1,i}]_j }{2\epsilon^{1.5}}\right| + 2 \alpha^2 \left\| \frac{1}{N} \sum_{i=1}^N \frac{g_{t,i}}{\sqrt{u_{t,i}}} \right\|^2 \nonumber \\
\leq & 2 \alpha^2 \left ( \frac{\beta_1}{1-\beta_1} \right)^2    G_{\infty} ^2 \frac{1}{N} \sum_{i=1}^N \sum_{j=1}^d   \frac{1}{{2\epsilon^2}}\left|{[\tilde u_{t,i}]_j - [\tilde u_{t-1,i}]_j }\right| + 2 \alpha^2 \left\| \frac{1}{N} \sum_{i=1}^N \frac{g_{t,i}}{\sqrt{u_{t,i}}} \right\|^2 \nonumber \\
= & 2 \alpha^2 \left ( \frac{\beta_1}{1-\beta_1} \right)^2    G_{\infty} ^2 \frac{1}{N}   \frac{1}{{2\epsilon^2}}\|\tilde U_t - \tilde U_{t-1} \|_{abs} + 2 \alpha^2 \left\| \frac{1}{N} \sum_{i=1}^N \frac{g_{t,i}}{\sqrt{u_{t,i}}} \right\|^2 \, ,
\end{align}
where the last inequality is again due to the definition that $[\tilde u_{t,i}]_j = \max ([ u_{t,i}]_j ,\epsilon )$ and the fact that $\max(\cdot, \epsilon)$ is 1-Lipschitz. 
Then, we have
\begin{align}
& \sum_{t=1}^T \mathbb E [\| Z_{t+1}-  Z_{t}\|^2]  \nonumber  \\
\leq & 2 \alpha^2 \left ( \frac{\beta_1}{1-\beta_1} \right)^2    G_{\infty} ^2 \frac{1}{N}   \frac{1}{{2\epsilon^2}}  \mathbb E \left [\sum_{t=1}^T    \|\tilde U_t - \tilde U_{t-1} \|_{abs} \right] +  2 \alpha^2  \sum_{t=1}^T   \mathbb E \left[ \left\| \frac{1}{N} \sum_{i=1}^N \frac{g_{t,i}}{\sqrt{u_{t,i}}} \right\|^2 \right] \nonumber  \\
\leq &  \alpha^2 \left ( \frac{\beta_1}{1-\beta_1} \right)^2   \frac{ G_{\infty} ^2 }{\sqrt{N}}   \frac{1}{{\epsilon^2}} \frac{1}{1-\lambda}  \mathbb E \left [ \sum_{t=1}^T  \| ( - \hat V_{t-2} + \hat V_{t-1}) \|_{abs}\right] + 2 \alpha^2 \sum_{t=1}^T  \mathbb E\left[ \left\| \frac{1}{N} \sum_{i=1}^N \frac{g_{t,i}}{\sqrt{u_{t,i}}} \right\|^2 \right] \, , \nonumber 
\end{align}
where the last inequality is due to \eqref{eq: diff_u_t}.

We now bound the last term on RHS of the above inequality. A trivial bound can be
\begin{align}\notag
& \sum_{t=1}^T \left\| \frac{1}{N} \sum_{i=1}^N \frac{g_{t,i}}{\sqrt{u_{t,i}}} \right\|^2  
\leq  \sum_{t=1}^T d G_{\infty}^2 \frac{1}{\epsilon}\, ,
\end{align}
due to $\|g_{t,i}\| \leq G_{\infty}$ and $[u_{t,i}]_j \geq \epsilon$, for all $j$ (verified from update rule of  $u_{t,i}$ and the assumption that $[v_{t,i}]_j \geq \epsilon$, for all $i$). 
However, the above bound is independent of $N$, to get a better bound, we need a more involved analysis to show its dependency on $N$. To do this, we first notice that
\begin{align}
&\mathbb E_{G_t| G_{1:t-1}} \left[ \left\| \frac{1}{N} \sum_{i=1}^N \frac{g_{t,i}}{\sqrt{u_{t,i}}} \right\|^2 \right] \nonumber  \\
= & \mathbb E_{G_t| G_{1:t-1}} \left[  \frac{1}{N^2} \sum_{i=1}^N 
\sum_{j=1}^N \left \langle \frac{\nabla f_i(x_{t,i}) + \xi_{t,i }}{\sqrt{u_{t,i}}}, \frac{\nabla f_j(x_{t,j}) + \xi_{t,j }}{\sqrt{u_{t,j}}} \right \rangle \right]  \nonumber  \\
\overset{(a)}{=}&\mathbb E_{G_t| G_{1:t-1}} \left[ \left\| \frac{1}{N} \sum_{i=1}^N \frac{\nabla f_i(x_{t,i})}{\sqrt{u_{t,i}}} \right\|^2 \right]  +  \mathbb E_{G_t| G_{1:t-1}} \left[  \frac{1}{N^2} \sum_{i=1}^N
\left \| \frac{ \xi_{t,i }}{\sqrt{u_{t,i}}}\right \|^2 \right] \nonumber  \\
\overset{(b)}{=}&  \left\| \frac{1}{N} \sum_{i=1}^N \frac{\nabla f_i(x_{t,i})}{\sqrt{u_{t,i}}} \right\|^2   +  \frac{1}{N^2} \sum_{i=1}^N  \sum_{l=1}^d
\frac{ \mathbb E_{G_t| G_{1:t-1}} [[\xi_{t,i}]_l^2] }{[u_{t,i}]_l}  \nonumber  \\
\overset{(c)}{\leq} & \left\| \frac{1}{N} \sum_{i=1}^N \frac{\nabla f_i(x_{t,i})}{\sqrt{u_{t,i}}} \right\|^2   +    \frac{d}{N}  
\frac{ \sigma^2 }{\epsilon} \, ,\nonumber 
\end{align}
where (a) is due to $\mathbb E_{G_t | G_{1:t-1}} [\xi_{t,i}] = 0 $ and $\xi_{t,i}$ is independent of $x_{t,j}$, $u_{t,j}$ for all $j$, and $ \xi_{j}$, for all $j \neq i$, (b) comes from the fact that $x_{t,i}$, $u_{t,i}$ are fixed given $G_{1:t}$, (c) is due to $\mathbb E_{G_t| G_{1:t-1}} [[\xi_{t,i}]_l^2 \leq \sigma^2$ and $[u_{t.i}]_l \geq \epsilon$ by definition.
Then we have 
\begin{align}\label{eq: split_var}
\mathbb E\left[ \left\| \frac{1}{N} \sum_{i=1}^N \frac{g_{t,i}}{\sqrt{u_{t,i}}} \right\|^2 \right] =  & \mathbb E_{ G_{1:t-1}} \left[ \mathbb E_{G_t| G_{1:t-1}} \left[ \left\| \frac{1}{N} \sum_{i=1}^N \frac{g_{t,i}}{\sqrt{u_{t,i}}} \right\|^2 \right] \right] \nonumber \\
\leq & \mathbb E_{ G_{1:t-1}} \left[  \left\| \frac{1}{N} \sum_{i=1}^N \frac{\nabla f_i(x_{t,i})}{\sqrt{u_{t,i}}} \right\|^2   +    \frac{d}{N}  
\frac{ \sigma^2 }{\epsilon} \right] \nonumber \\
= &  \mathbb E \left[  \left\| \frac{1}{N} \sum_{i=1}^N \frac{\nabla f_i(x_{t,i})}{\sqrt{u_{t,i}}} \right\|^2     \right] + \frac{d}{N}  
\frac{ \sigma^2 }{\epsilon}  \, .
\end{align}

\noindent In standard analysis of SGD-like distributed algorithms, the term corresponding to $ \mathbb E \left[  \left\| \frac{1}{N} \sum_{i=1}^N \frac{\nabla f_i(x_{t,i})}{\sqrt{u_{t,i}}} \right\|^2     \right] $ will be merged with the first order descent when the stepsize is chosen to be small enough. However, in our case, the term cannot be merged because it is different from the first order descent in our algorithm. A brute-force upper bound is possible but this will lead to a worse convergence rate in terms of $N$. Thus, we need a more detailed analysis for the term in the following.

\newpage

\begin{align}\notag
\mathbb E \left[  \left\| \frac{1}{N} \sum_{i=1}^N \frac{\nabla f_i(x_{t,i})}{\sqrt{u_{t,i}}} \right\|^2     \right]  
=& \mathbb E \left[  \left\|\frac{1}{N} \sum_{i=1}^N \frac{\nabla f_i(x_{t,i})}{\sqrt{\overline U_t}  } +   \frac{1}{N} \sum_{i=1}^N \nabla f_i(x_{t,i}) \odot \left( \frac{1}{\sqrt{u_{t,i}}} - \frac{1}{\sqrt{\overline U_{t}}} \right) \right\|^2     \right] \nonumber  \\
\leq & 2 \mathbb E \left[  \left\|\frac{1}{N} \sum_{i=1}^N \frac{\nabla f_i(x_{t,i})}{\sqrt{\overline U_t}  } \right\|^2 \right] + 2\mathbb E \left[ \left\|  \frac{1}{N} \sum_{i=1}^N \nabla f_i(x_{t,i}) \odot \left( \frac{1}{\sqrt{u_{t,i}}} - \frac{1}{\sqrt{\overline U_{t}}} \right) \right\|^2     \right] \nonumber \\
\leq & 2 \mathbb E \left[  \left\|\frac{1}{N} \sum_{i=1}^N \frac{\nabla f_i(x_{t,i})}{\sqrt{\overline U_t}  } \right\|^2 \right] + 2\mathbb E \left[  \frac{1}{N} \sum_{i=1}^N \left\|    \nabla f_i(x_{t,i}) \odot \left( \frac{1}{\sqrt{u_{t,i}}} - \frac{1}{\sqrt{\overline U_{t}}} \right) \right\|^2     \right] \nonumber \\
\leq & 2 \mathbb E \left[  \left\|\frac{1}{N} \sum_{i=1}^N \frac{\nabla f_i(x_{t,i})}{\sqrt{\overline U_t}  } \right\|^2 \right] + 2\mathbb E \left[  \frac{1}{N} \sum_{i=1}^N G_{\infty}^2  \frac{1}{\sqrt{\epsilon}}\left\|     \frac{1}{\sqrt{u_{t,i}}} - \frac{1}{\sqrt{\overline U_{t}}}  \right\|_1     \right] \, .\nonumber
\end{align}

Summing over $T$, we have
\begin{align}\label{eq: variance_bound_1}
&\sum_{t=1}^T \mathbb E \left[  \left\| \frac{1}{N} \sum_{i=1}^N \frac{\nabla f_i(x_{t,i})}{\sqrt{u_{t,i}}} \right\|^2     \right]   \nonumber \\
\leq & 2\sum_{t=1}^T \mathbb E \left[  \left\|\frac{1}{N} \sum_{i=1}^N \frac{\nabla f_i(x_{t,i})}{\sqrt{\overline U_t}  } \right\|^2 \right] + 2 \sum_{t=1}^T \mathbb E \left[  \frac{1}{N} \sum_{i=1}^N G_{\infty}^2  \frac{1}{\sqrt{\epsilon}}\left\|     \frac{1}{\sqrt{u_{t,i}}} - \frac{1}{\sqrt{\overline U_{t}}}  \right\|_1     \right] \, .
\end{align}
For the last term on RHS of \eqref{eq: variance_bound_1}, we can bound it similarly as what we did for $D_2$ from \eqref{eq:T_2_bound} to \eqref{eq: T_6_bound}, which yields
\begin{align}\label{eq: diff_u}
\sum_{t=1}^T \mathbb E \left[  \frac{1}{N} \sum_{i=1}^N G_{\infty}^2  \frac{1}{\sqrt{\epsilon}}\left\|     \frac{1}{\sqrt{u_{t,i}}} - \frac{1}{\sqrt{\overline U_{t}}}  \right\|_1     \right]&  \leq  \sum_{t=1}^T \mathbb E \left[  \frac{1}{N} \sum_{i=1}^N G_{\infty}^2  \frac{1}{\sqrt{\epsilon}} \frac{1}{2\epsilon^{1.5}} \left\|  u_{t,i} -    \overline U_{t}  \right\|_1     \right] \nonumber \\
=& \sum_{t=1}^T \mathbb E \left[  \frac{1}{N}  G_{\infty}^2 \frac{1}{2\epsilon^2} \left\|     \overline U_{t} \mathbf 1^T - U_{t}  \right\|_{abs}    \right]  \nonumber \\
\leq & \sum_{t=1}^T \mathbb E \left[  \frac{1}{N}  G_{\infty}^2 \frac{1}{2\epsilon^2} \| - \sum_{l=2}^N   \tilde U_t q_l q_l^T \|_{abs}    \right] \nonumber \\ 
\leq & \frac{1}{\sqrt{N}}  G_{\infty}^2 \frac{1}{2\epsilon^2}   \mathbb E \left[   \sum_{o=0}^{T-1} \frac{\lambda}{1-\lambda}     \|    (- \hat V_{o-1} + \hat V_{o} ) \|_{abs}    \right] \, .
\end{align}
Further, we have 
\begin{align}
&\sum_{t=1}^T \mathbb E \left[  \left\|\frac{1}{N} \sum_{i=1}^N \frac{\nabla f_i(x_{t,i})}{\sqrt{\overline U_t}  } \right\|^2 \right]  \nonumber \\
\leq & 2 \sum_{t=1}^T \mathbb E \left[  \left\|\frac{1}{N} \sum_{i=1}^N \frac{\nabla f_i(\overline X_{t})}{\sqrt{\overline U_t}  } \right\|^2 \right] + 2 \sum_{t=1}^T \mathbb E \left[  \left\|\frac{1}{N} \sum_{i=1}^N \frac{\nabla f_i(\overline X_t) - \nabla f_i(x_{t,i})}{\sqrt{\overline U_t}  } \right\|^2 \right] \nonumber \\
= & 2 \sum_{t=1}^T \mathbb E \left[  \left\| \frac{\nabla f(\overline X_{t})}{\sqrt{\overline U_t}  } \right\|^2 \right] + 2 \sum_{t=1}^T \mathbb E \left[  \left\|\frac{1}{N} \sum_{i=1}^N \frac{\nabla f_i(\overline X_t) - \nabla f_i(x_{t,i})}{\sqrt{\overline U_t}  } \right\|^2 \right]  \nonumber
\end{align}
and the last term on RHS of the above inequality can be bounded following similar procedures from \eqref{eq: T_3_bound_first} to \eqref{eq: T_5_bound}, as what we did for $D_3$. Completing the procedures yields
\begin{align}\label{eq: diff_g}
\sum_{t=1}^T \mathbb E \left[  \left\|\frac{1}{N} \sum_{i=1}^N \frac{\nabla f_i(\overline X_t) - \nabla f_i(x_{t,i})}{\sqrt{\overline U_t}  } \right\|^2 \right] \leq & \sum_{t=1}^T \mathbb E \left [ L \frac{1}{\epsilon} \frac{1}{N} \sum_{i=1}^N \left\|x_{t,i} - \overline X_t \right\|^2 \right] \nonumber  \\
\leq & \sum_{t=1}^T \mathbb E \left [ L \frac{1}{\epsilon} \frac{1}{N} \alpha^2 \left( \frac{1}{1-\lambda}\right)Nd G_{\infty}^2 \frac{1}{\epsilon} \right] \nonumber \\
= & T L \frac{1}{\epsilon^2}  \alpha^2 \left( \frac{1}{1-\lambda}\right)d G_{\infty}^2 \, .
\end{align}

Finally, combining \eqref{eq: split_var} to \eqref{eq: diff_g}, we obtain
\begin{align}
 \sum_{t=1}^T \mathbb E\left[ \left\| \frac{1}{N} \sum_{i=1}^N \frac{g_{t,i}}{\sqrt{u_{t,i}}} \right\|^2 \right] 
\leq &  4 \sum_{t=1}^T \mathbb E \left[  \left\| \frac{\nabla f(\overline X_{t})}{\sqrt{\overline U_t}  } \right\|^2 \right] + 4 T L \frac{1}{\epsilon^2}  \alpha^2 \left( \frac{1}{1-\lambda}\right)d G_{\infty}^2  \nonumber \\
& +  2\frac{1}{\sqrt{N}}  G_{\infty}^2 \frac{1}{2\epsilon^2}   \mathbb E \left[   \sum_{o=0}^{T-1} \frac{\lambda}{1-\lambda}     \|    (- \hat V_{o-1} + \hat V_{o} ) \|_{abs}    \right]  + T \frac{d}{N}
\frac{ \sigma^2 }{\epsilon} \nonumber  \\
\leq &  4 \frac{1}{\sqrt{\epsilon}} \sum_{t=1}^T \mathbb E \left[  \left\| \frac{\nabla f(\overline X_{t})}{\overline U_t^{1/4}  } \right\|^2 \right] + 4 T L \frac{1}{\epsilon^2}  \alpha^2 \left( \frac{1}{1-\lambda}\right)d G_{\infty}^2 \nonumber   \\
& +  2\frac{1}{\sqrt{N}}  G_{\infty}^2 \frac{1}{2\epsilon^2}   \mathbb E \left[   \sum_{o=0}^{T-1} \frac{\lambda}{1-\lambda}     \|    (- \hat V_{o-1} + \hat V_{o} ) \|_{abs}    \right]  + T \frac{d}{N}
\frac{ \sigma^2 }{\epsilon}. \nonumber 
\end{align}
where the last inequality is due to each element of $\overline U_t$ is lower bounded by $\epsilon$ by definition.

Combining all above, we obtain
\begingroup
\allowdisplaybreaks
\begin{align}\label{eq: final_bound}
 &\frac{1}{T}\sum_{t=1}^T  \mathbb E \left [\left\|\frac{\nabla f( \overline X_{t})}{\overline U_{t}^{1/4}}\right\|^2  \right] \nonumber  \\
 \leq & \frac{2}{T\alpha} ( \mathbb E  [f( Z_{1})]  -  \mathbb E [f( Z_{T+1})] 
+ \frac{L}{T}   \alpha \left ( \frac{\beta_1}{1-\beta_1} \right)^2   \frac{ G_{\infty} ^2 }{\sqrt{N}}   \frac{1}{{\epsilon^2}} \frac{1}{1-\lambda}  \mathbb E \left[ \mathcal{V}_T \right]   \nonumber  \\
& + \frac{8L}{T}\alpha\frac{1}{\sqrt{\epsilon}} \sum_{t=1}^T \mathbb E \left[  \left\| \frac{\nabla f(\overline X_{t})}{{\overline U_t^{1/4}}  } \right\|^2 \right] + {8L^2}\alpha  \frac{1}{\epsilon^2}  \alpha^2 \left( \frac{1}{1-\lambda}\right)d G_{\infty}^2 \nonumber \\
& +  \frac{4L}{T}\alpha \frac{1}{\sqrt{N}}  G_{\infty}^2 \frac{1}{2\epsilon^2}   \mathbb E \left[   \sum_{o=0}^{T-1} \frac{\lambda}{1-\lambda}     \|    (- \hat V_{o-1} + \hat V_{o} ) \|_{abs}    \right]  +  {2L}\alpha  \frac{d}{N}
\frac{ \sigma^2 }{\epsilon} \nonumber  \\
&+ \frac{2}{T}\frac{\beta_1}{1-\beta_1} G_{\infty}^2 \frac{1}{2\epsilon^{1.5}} \frac{1}{\sqrt{N}}   \mathbb E \left [  \frac{1}{1-\lambda}   \mathcal{V}_T   \right] + \frac{2}{T} \frac{G_{\infty}^2}{\sqrt{N}}\frac{1}{2\epsilon^{1.5}} \frac{\lambda}{1-\lambda}   \mathbb E \left[ \mathcal{V}_T \right]\nonumber   \\
& + \frac{3}{T} \left ( \sum_{t=1} ^TL\left (\frac{1}{1-\lambda} \right)^2 \alpha ^ 2d G_{\infty}^2 \frac{1}{\epsilon^{1.5}} + \sum_{t=1}^T L\left( \frac{\beta_1}{1-\beta_1}\right)^2 \alpha^2 d \frac{G_{\infty}^2}{\epsilon^{1.5}}\right) \nonumber \\
= &  \frac{2}{T\alpha} ( \mathbb E  [f( Z_{1})]  -  \mathbb E [f( Z_{T+1})]) +  {2L}\alpha  \frac{d}{N}
\frac{ \sigma^2 }{\epsilon} + {8L}\alpha\frac{1}{\sqrt{\epsilon}} \frac{1}{T} \sum_{t=1}^T \mathbb E \left[  \left\| \frac{\nabla f(\overline X_{t})}{{\overline U_t^{1/4}}  } \right\|^2 \right] \nonumber \\
&+  3\alpha^2 d \left(\left( \frac{\beta_1}{1-\beta_1}\right)^2 + \left (\frac{1}{1-\lambda} \right)^2 \right)L  \frac{G_{\infty}^2 }{\epsilon^{1.5}} +  {8}\alpha^3 L^2     \left( \frac{1}{1-\lambda}\right)d \frac{G_{\infty}^2}{\epsilon^2} \nonumber \\
& +   \frac{1}{T \epsilon^{1.5}}  \frac{G_{\infty}^2}{\sqrt{N}} \frac{1}{1-\lambda}  \left( L  \alpha \left ( \frac{\beta_1}{1-\beta_1} \right)^2     \frac{1}{{\epsilon^{0.5}}}  +   \lambda + \frac{\beta_1}{1-\beta_1} + 2L\alpha \frac{1}{\epsilon^{0.5}}\lambda   \right)   \mathbb E \left[ \mathcal{V}_T \right] \, .
\end{align}

where $ \mathcal{V}_T : = \sum_{t=1}^{T}   \|    (- \hat V_{t-2} + \hat V_{t-1} ) \|_{abs}$.
Set $\alpha = \frac{1}{\sqrt{dT}}$ and when $\alpha  \leq \frac{\epsilon^{0.5}}{16L} $, we further have
\begin{align}\notag
& \frac{1}{T}\sum_{t=1}^T  \mathbb E \left [\left\|\frac{\nabla f( \overline X_{t})}{\overline U_{t}^{1/4}}\right\|^2  \right] \nonumber \\
\leq & \frac{4}{T\alpha} ( \mathbb E  [f( Z_{1})]  -  \mathbb E [f( Z_{T+1})]) +  {4L}\alpha  \frac{d}{N}
\frac{ \sigma^2 }{\epsilon}  \nonumber \\
&+  6\alpha^2 d \left(\left( \frac{\beta_1}{1-\beta_1}\right)^2 + \left (\frac{1}{1-\lambda} \right)^2 \right)L  \frac{G_{\infty}^2 }{\epsilon^{1.5}} +  {16}\alpha^3 L^2     \left( \frac{1}{1-\lambda}\right)d \frac{G_{\infty}^2}{\epsilon^2} \nonumber \\
& +   \frac{2}{T \epsilon^{1.5}}  \frac{G_{\infty}^2}{\sqrt{N}} \frac{1}{1-\lambda}  \left( L  \alpha \left ( \frac{\beta_1}{1-\beta_1} \right)^2     \frac{1}{{\epsilon^{0.5}}}  +   \lambda + \frac{\beta_1}{1-\beta_1} + 2L\alpha \frac{1}{\epsilon^{0.5}}\lambda   \right)   \mathbb E \left[ \mathcal{V}_T \right] \nonumber \\
\leq & \frac{4}{T\alpha} ( \mathbb E  [f( Z_{1})]  -  \min_x  f(x)) +  {4L}\alpha  \frac{d}{N}
\frac{ \sigma^2 }{\epsilon}  \nonumber \\
&+  6\alpha^2 d \left(\left( \frac{\beta_1}{1-\beta_1}\right)^2 + \left (\frac{1}{1-\lambda} \right)^2 \right)L  \frac{G_{\infty}^2 }{\epsilon^{1.5}} +  {16}\alpha^3d L^2     \left( \frac{1}{1-\lambda}\right) \frac{G_{\infty}^2}{\epsilon^2} \nonumber \\
& +   \frac{2}{T \epsilon^{1.5}}  \frac{G_{\infty}^2}{\sqrt{N}} \frac{1}{1-\lambda}  \left( L  \alpha \left ( \frac{\beta_1}{1-\beta_1} \right)^2     \frac{1}{{\epsilon^{0.5}}}  +   \lambda + \frac{\beta_1}{1-\beta_1} + 2L\alpha \frac{1}{\epsilon^{0.5}}\lambda   \right)   \mathbb E \left[ \mathcal{V}_T \right] \nonumber \\
\leq & C_1\left(\frac{1}{T\alpha} ( \mathbb E  [f( Z_{1})]  -  \min_x  f(x)) +  \alpha  \frac{d\sigma^2}{N}\right)
+  C_2 \alpha^2 d  +  C_3 \alpha^3d + \frac{1}{T\sqrt{N}} (C_4 +  C_5 \alpha) \mathbb E \left[ \mathcal{V}_T \right]   \nonumber \\
%= & \frac{4\sqrt{d}}{\sqrt{T}} ( \mathbb E  [f( Z_{1})]  -  \mathbb E [f( Z_{T+1})]) +  {4L}\frac{\sqrt{d}}{\sqrt{T}}  \frac{1}{N}
%\frac{ \sigma^2 }{\epsilon}  \nonumber \\
%&+  6
%\frac{1}{T}  \left(\left( \frac{\beta_1}{1-\beta_1}\right)^2 + \left (\frac{1}{1-\lambda} \right)^2 \right)L  \frac{G_{\infty}^2 }{\epsilon^{1.5}} +  {16}\frac{1}{T^{1.5}d^{0.5}}L^2     \left( \frac{1}{1-\lambda}\right) \frac{G_{\infty}^2}{\epsilon^2} \nonumber \\
%& +   \frac{2}{T \epsilon^{1.5}}  \frac{G_{\infty}^2}{\sqrt{N}} \frac{1}{1-\lambda}  \left( \frac{L}{\sqrt{Td}}  \left ( \frac{\beta_1}{1-\beta_1} \right)^2     \frac{1}{{\epsilon^{0.5}}}  +   \lambda + \frac{\beta_1}{1-\beta_1} + 2\frac{L}{\sqrt{Td}} \frac{1}{\epsilon^{0.5}}\lambda   \right)   \mathbb E \left[ \mathcal{V}_T \right] \nonumber \\
%\leq  & C_1 \frac{\sqrt{d}}{\sqrt{T}} \left(\mathbb E  [f( Z_{1})]  - \min_{z} f(z)  + \frac{\sigma^2}{N}\right)  +  \frac{ 1}{T} C_2    + \frac{1}{T^{1.5}d^{0.5}} C_3 \nonumber \\ 
%&+ \left(  \frac{1}{T N^{0.5} } C_4 + \frac{1}{T^{1.5}d^{0.5} N^{0.5}} C_5  \right) \mathbb E \left[ \mathcal{V}_T \right]  \, ,\nonumber
\end{align}
where the first inequality is obtained by moving the term ${8L}\alpha\frac{1}{\sqrt{\epsilon}} \frac{1}{T} \sum_{t=1}^T \mathbb E \left[  \left\| \frac{\nabla f(\overline X_{t})}{{\overline U_t^{1/4}}  } \right\|^2 \right] $ on the RHS of \eqref{eq: final_bound} to the LHS to cancel it using the assumption ${8L}\alpha\frac{1}{\sqrt{\epsilon}} \leq \frac{1}{2} $ followed by multiplying both sides by 2.
The constants introduced in the last step are defined as following
\begin{align}
C_1 = & \max (4, 4{L/\epsilon}) \nonumber \, ,\\
C_2 = & 6 \left(\left( \frac{\beta_1}{1-\beta_1}\right)^2 + \left (\frac{1}{1-\lambda} \right)^2 \right)L  \frac{G_{\infty}^2 }{\epsilon^{1.5}}   \nonumber \, ,\\
C_3 = & 16L^2 \left ( \frac{1}{1-\lambda}\right) \frac{G_{\infty}^2}{\epsilon^2} \nonumber \, ,\\
C_4 = &  \frac{2}{ \epsilon^{1.5}}   \frac{1}{1-\lambda}  \left(     \lambda + \frac{\beta_1}{1-\beta_1}    \right){G_{\infty}^2} \nonumber \, ,\\
C_5 = &  \frac{2}{ \epsilon^{2}}   \frac{1}{1-\lambda}   L   \left ( \frac{\beta_1}{1-\beta_1} \right)^2 {G_{\infty}^2}  + \frac{4}{ \epsilon^{2}}   \frac{\lambda}{1-\lambda}   L    {G_{\infty}^2} \nonumber \, .
\end{align}
Substituting into $Z_1 = \overline X_1$ completes the proof. \hfill $\square$

\clearpage

\section{Proof of Theorem~\ref{thm: dams_converge}} \label{app: proof_ams}

Under some assumptions stated in Corollary~\ref{corl: adm_convergence}, we have that
\begin{align}\label{eq: rep_thm1}
	 \frac{1}{T}\sum_{t=1}^T  \mathbb E \left [\left\|\frac{\nabla f( \overline X_{t})}{\overline U_{t}^{1/4}}\right\|^2  \right] 
	\leq  & C_1 \frac{\sqrt{d}}{\sqrt{TN}} \left(( \mathbb E  [f( Z_{1})]  -  \min_x  f(x)) +    \sigma^2 \right)  +  C_2 \frac{N}{T}  +  C_3 \frac{N^{1.5}}{T^{1.5}d^{0.5}} 
	\nonumber \\
    &+  \left(C_4 \frac{1}{T\sqrt{N}} +  C_5   \frac{1}{T^{1.5}d^{0.5}}\right)\mathbb E \left[ \sum_{t=1}^{T}   \|    (- \hat V_{t-2} + \hat V_{t-1} ) \|_{abs} \right] 
\end{align}
where $\| \cdot\|_{abs}$  denotes the entry-wise $L_1$ norm of a matrix (i.e $\| A\|_{abs} = \sum_{i,j}{|A_{ij}|}$) and $C_1, C_2 ,C_3, C_4, C_5$ are defined in Theorem~\ref{thm: dagm_converge}. 

Since Algorithm~\ref{alg: damsgrad} is a special case of~\ref{alg: dadaptive}, building on result of Theorem~\ref{thm: dagm_converge}, we just need to characterize the growth speed of $\mathbb E \left[ \sum_{t=1}^{T}   \|    (- \hat V_{t-2} + \hat V_{t-1} ) \|_{abs} \right]  $ to prove convergence of Algorithm~\ref{alg: damsgrad}.  By the update rule of Algorithm~\ref{alg: damsgrad}, we know $\hat V_t$ is non decreasing and thus
\begin{align}
\mathbb E \left[ \sum_{t=1}^{T}   \|    (- \hat V_{t-2} + \hat V_{t-1} ) \|_{abs} \right] = &\mathbb E \left[ \sum_{t=1}^{T}  \sum_{i=1}^N \sum_{j=1}^d    |- [\hat v_{t-2,i}]_j + [\hat v_{t-1,i}]_j | \right]  \nonumber  \\
= &\mathbb E \left[ \sum_{t=1}^{T}  \sum_{i=1}^N \sum_{j=1}^d    (- [\hat v_{t-2,i}]_j + [\hat v_{t-1,i}]_j ) \right]  \nonumber  \\
= &\mathbb E \left[   \sum_{i=1}^N \sum_{j=1}^d    (- [\hat v_{-1,i}]_j + [\hat v_{T-1,i}]_j ) \right]   \nonumber \\
= &\mathbb E \left[   \sum_{i=1}^N \sum_{j=1}^d    (- [\hat v_{0,i}]_j + [\hat v_{T-1,i}]_j ) \right] \, , 
\end{align}
where the last equality is because  we defined $\hat V_{-1} \triangleq \hat V_0$  previously.

Further, because $\|g_{t,i}\|_{\infty} \leq G_{\infty}$ for all $t,i$ and $v_{t,i}$ is a exponential moving average of $g_{k,i}^2, k=1,2,\cdots,t$, we know $|[v_{t,i}]_j| \leq G^2_{\infty}$, for all $t,i,j$. In addition, by update rule of $\hat V_t$, we also know each element of $\hat V_{t}$ also cannot be greater than $G^2_{\infty}$, i.e., $|[\hat v_{t,i}]_j| \leq G^2_{\infty}$, for all $t,i,j$. 
Given the fact that $[\hat v_{0,i}]_j \geq 0$ , we have 
\begin{align}\notag
\mathbb E \left[ \sum_{t=1}^{T}   \|    (- \hat V_{t-2} + \hat V_{t-1} ) \|_{abs} \right]  = \mathbb E \left[   \sum_{i=1}^N \sum_{j=1}^d    (- [\hat v_{0,i}]_j + [\hat v_{T-1,i}]_j ) \right]  
\leq  \mathbb E \left[   \sum_{i=1}^N \sum_{j=1}^d  G_{\infty}^2   \right] =  Nd G_{\infty}^2 \, .
\end{align}
Substituting the above into \eqref{eq: rep_thm1}, we have 
\begin{align}\label{eq: sub_thm1}
	 \frac{1}{T}\sum_{t=1}^T  \mathbb E \left [\left\|\frac{\nabla f( \overline X_{t})}{\overline U_{t}^{1/4}}\right\|^2  \right] 
	\leq  & C_1 \frac{\sqrt{d}}{\sqrt{TN}} \left(( \mathbb E  [f( Z_{1})]  -  \min_x  f(x)) +    \sigma^2 \right)  +  C_2 \frac{N}{T}  +  C_3 \frac{N^{1.5}}{T^{1.5}d^{0.5}} \nonumber
	 \\
    &+  \left(C_4 \frac{1}{T\sqrt{N}} +  C_5   \frac{1}{T^{1.5}d^{0.5}}\right)NdG_{\infty}^2  \nonumber \\
    =  & C_1' \frac{\sqrt{d}}{\sqrt{TN}} \left(( \mathbb E  [f( Z_{1})]  -  \min_x  f(x)) +    \sigma^2 \right)  +  C_2' \frac{N}{T}  +  C_3' \frac{N^{1.5}}{T^{1.5}d^{0.5}} 
	 \nonumber \\
    &+  C_4' \frac{\sqrt{N}d}{T} +  C_5'  \frac{Nd^{0.5}}{T^{1.5}}\,,
\end{align}
where we have 
\begin{align}
C_1' = C_1 \quad C_2' = C_2 \quad C_3' = C_3 \quad C_4' = C_4G_{\infty}^2 \quad C_5' = C_5 G_{\infty}^2 \, .
\end{align}
and we conclude the proof. \hfill $\square$

\vspace{0.2in}

\section{Proof of Theorem~\ref{thm: dadagrad_converge}} \label{app: proof_adagrad}

The proof follows the same flow as that of Theorem~\ref{thm: dams_converge}. Under assumptions stated in Corollary~\ref{corl: adm_convergence}, set $\alpha = \sqrt{N}/\sqrt{Td}$, we have that
\begin{align}\label{eq: rep_thm1bis}
	 \frac{1}{T}\sum_{t=1}^T  \mathbb E \left [\left\|\frac{\nabla f( \overline X_{t})}{\overline U_{t}^{1/4}}\right\|^2  \right] 
	\leq  & C_1 \frac{\sqrt{d}}{\sqrt{TN}} \left(( \mathbb E  [f( Z_{1})]  -  \min_x  f(x)) +    \sigma^2 \right)  +  C_2 \frac{N}{T}  +  C_3 \frac{N^{1.5}}{T^{1.5}d^{0.5}} 
	\nonumber \\
    &+  \left(C_4 \frac{1}{T\sqrt{N}} +  C_5   \frac{1}{T^{1.5}d^{0.5}}\right) 
\mathbb E \left[ \sum_{t=1}^{T}   \|    (- \hat V_{t-2} + \hat V_{t-1} ) \|_{abs} \right]  \, ,
\end{align}
where $\| \cdot\|_{abs}$  denotes the entry-wise $L_1$ norm of a matrix (i.e $\| A\|_{abs} = \sum_{i,j}{|A_{ij}|}$) and $C_1, C_2 ,C_3, C_4, C_5$ are defined in Theorem~\ref{thm: dagm_converge}. 

Again, since decentralized AdaGrad is a special case of~\ref{alg: dadaptive}, Corollary~\ref{corl: adm_convergence} applies and what we need is to upper bound  $\mathbb E \left[ \sum_{t=1}^{T}   \|    (- \hat V_{t-2} + \hat V_{t-1} ) \|_{abs} \right]  $  derive convergence rate.  By the update rule of decentralized AdaGrad, we have $\hat v_{t,i} = \frac{1}{t}( \sum_{k=1}^{t}g_{k,i}^2)$ for $t \geq 1$ and $\hat v_{0,i} = \epsilon \mathbf 1$. Then we have for $t \geq 3$,
\begin{align}
&\mathbb E \left[ \sum_{t=1}^{T}   \|    (- \hat V_{t-2} + \hat V_{t-1} ) \|_{abs} \right]  \nonumber \\
= &\mathbb E \left[ \sum_{t=1}^{T}  \sum_{i=1}^N \sum_{j=1}^d    |- [\hat v_{t-2,i}]_j + [\hat v_{t-1,i}]_j | \right] \nonumber   \\
\leq &\mathbb E \left[ \sum_{t=3}^{T}  \sum_{i=1}^N \sum_{j=1}^d    |- \frac{1}{t-2}( [\sum_{k=1}^{t-2}g_{k,i}^2]_j) + \frac{1}{t-1}([ \sum_{k=1}^{t-1}g_{k,i}^2]_j )| \right] + Nd (G_{\infty}^2 - \epsilon) \nonumber \\
\leq &\mathbb E \left[ \sum_{t=3}^{T}  \sum_{i=1}^N \sum_{j=1}^d    | (\frac{1}{t-1} - \frac{1}{t-2})( [\sum_{k=1}^{t-2}g_{k,i}^2]_j) + \frac{1}{t-1} [g_{t-1,i}^2]_j )| \right] +  Nd G_{\infty}^2  \nonumber  \\
=  &\mathbb E \left[ \sum_{t=3}^{T}  \sum_{i=1}^N \sum_{j=1}^d    | (-\frac{1}{(t-1)(t-2)} )( [\sum_{k=1}^{t-2}g_{k,i}^2]_j) + \frac{1}{t-1}[g_{t-1,i}^2]_j | \right]    +  Nd G_{\infty}^2  \nonumber \\
\leq &\mathbb E \left[ \sum_{t=3}^{T}  \sum_{i=1}^N \sum_{j=1}^d    \max\left( \frac{1}{(t-1)(t-2)} ( [\sum_{k=1}^{t-2}g_{k,i}^2]_j) , \frac{1}{t-1}[g_{t-1,i}^2]_j \right) \right]    +  Nd G_{\infty}^2 \nonumber  \\
\leq  &\mathbb E \left[ N d \sum_{t=3}^{T}      \frac{G_{\infty}^2}{t-1}   \right]   +  Nd G_{\infty}^2 \nonumber  \\
\leq & Nd G_{\infty}^2 \log (T) +  Nd G_{\infty}^2 \nonumber \\
= & NdG_{\infty}^2 (\log (T) + 1) \nonumber
\end{align}
where the first equality is because  we defined $\hat V_{-1} \triangleq \hat V_0$  previously and $\|g_{k,i}\|_{\infty} \leq G_{\infty}$ by assumption.

%Further, because $\|g_{t,i}\|_{\infty} \leq G_{\infty}$ for all $t,i$ and $v_{t,i}$ is a exponential moving average of $g_{k,i}^2, k=1,2,\cdots,t$, we know $|[v_{t,i}]_j| \leq G^2_{\infty}$, for all $t,i,j$. In addition, by update rule of $\hat V_t$, we also know each element of $\hat V_{t}$ also cannot be greater than $G^2_{\infty}$, i.e., $|[\hat v_{t,i}]_j| \leq G^2_{\infty}$, for all $t,i,j$. 
%Given the fact that $[\hat v_{0,i}]_j \geq 0$ , we have 
%\begin{align}\notag
%\mathbb E \left[ \sum_{t=1}^{T}   \|    (- \hat V_{t-2} + \hat V_{t-1} ) \|_{abs} \right]  = \mathbb E \left[   \sum_{i=1}^N \sum_{j=1}^d    (- [\hat v_{0,i}]_j + [\hat v_{T-1,i}]_j ) \right]  
%\leq  \mathbb E \left[   \sum_{i=1}^N \sum_{j=1}^d  G_{\infty}^2   \right] =  Nd G_{\infty}^2 \, .
%\end{align}
Substituting the above into \eqref{eq: rep_thm1bis}, we have 
\begin{align}%\label{eq: sub_thm1bis}
	 \frac{1}{T}\sum_{t=1}^T  \mathbb E \left [\left\|\frac{\nabla f( \overline X_{t})}{\overline U_{t}^{1/4}}\right\|^2  \right] 
	\leq  & C_1 \frac{\sqrt{d}}{\sqrt{TN}} \left(( \mathbb E  [f( Z_{1})]  -  \min_x  f(x)) +    \sigma^2 \right)  +  C_2 \frac{N}{T}  +  C_3 \frac{N^{1.5}}{T^{1.5}d^{0.5}} 
	\nonumber \\
    &+  \left(C_4 \frac{1}{T\sqrt{N}} +  C_5   \frac{1}{T^{1.5}d^{0.5}}\right) 
NdG_{\infty}^2 (\log(T)+1) \nonumber \\
	=  & C_1' \frac{\sqrt{d}}{\sqrt{TN}} \left(( \mathbb E  [f( Z_{1})]  -  \min_x  f(x)) +    \sigma^2 \right)  +  C_2' \frac{N}{T}  +  C_3' \frac{N^{1.5}}{T^{1.5}d^{0.5}} 
	\nonumber \\
    &+  C_4' \frac{d\sqrt{N}(\log(T)+1)}{T} +  C_5'   \frac{(\log(T)+1)N\sqrt{d}}{T^{1.5}}\,, \nonumber
%=  & C_1' \frac{\sqrt{d}}{\sqrt{T}} \left(\mathbb E  [f( \overline X_{1})]  - \min_{z} f(z)  + \frac{\sigma^2}{N}\right)  +  \frac{ 1}{T} C_2'    + \frac{1}{T^{1.5}d^{0.5}} C_3'  \\ 
%&+   \frac{d(\log (T) + 1)}{T }\sqrt{N} C_4' + \frac{\sqrt{d}(\log (T) + 1)}{T^{1.5} } \sqrt{N} C_5'\, ,
\end{align}
where we have 
\begin{align}
C_1' = C_1 \quad C_2' = C_2 \quad C_3' = C_3 \quad C_4' = C_4G_{\infty}^2 \quad C_5' = C_5 G_{\infty}^2 \, .
\end{align}
and we conclude the proof. \hfill $\square$

%\clearpage

\vspace{0.2in}

\bibliographystyle{plainnat}
\bibliography{reference}

\end{document}

%% file: preamble/packages.tex
\usepackage{enumerate}

\usepackage{amsmath}
\usepackage{amsthm}
\usepackage{amssymb}
\usepackage{thmtools,thm-restate}
\usepackage{amsfonts}
\usepackage{mathtools}
\usepackage{enumitem}

\usepackage{tcolorbox}

\newtheorem{corollary}{Corollary}[theorem]
\usepackage{mathtools}

\let\counterwithin\relax
\usepackage{chngcntr}
\usepackage{apptools}
\AtAppendix{\counterwithin{lemma}{section}}
\AtAppendix{\counterwithin{theorem}{section}}
\AtAppendix{\counterwithin{algorithm}{section}}